\documentclass[letterpaper]{article} 
\usepackage{aaai25}  
\usepackage{times}  
\usepackage{helvet}  
\usepackage{courier}  
\usepackage[hyphens]{url}  
\usepackage{graphicx} 
\usepackage{natbib}  
\usepackage{caption} 
\usepackage{algorithm}
\usepackage{algorithmic}

\usepackage{newfloat}
\usepackage{listings}
\DeclareCaptionStyle{ruled}{labelfont=normalfont,labelsep=colon,strut=off} 
\floatstyle{ruled}
\newfloat{listing}{tb}{lst}{}
\floatname{listing}{Listing}
\usepackage{subcaption}
\usepackage{import}
\usepackage{amsmath} 
\usepackage{amsthm}

\newcommand{\mathbbm}[1]{\text{\usefont{U}{bbm}{m}{n}#1}}

\newtheorem{theorem}{Theorem}
\newtheorem{corollary}{Corollary}
\newtheorem{lemma}{Lemma}
\newtheorem{definition}{Definition}

\newtheorem{proposition}{Proposition}

\def\Pr{\mathop{\rm Pr}\nolimits}

\def\EE{{\cal E}}

\def\SS{{\cal S}}
\def\MM{{\cal M}}

\def\CC{{\cal C}}

\def\X{{\mathbf X}}
\def\Y{{\mathbf Y}}

\def\U{{\mathbf U}}
\def\W{{\mathbf W}}

\def\V{{\mathbf V}}
\def\P{{\mathbf P}}

\def\Q{{\mathbf Q}}
\def\x{{\mathbf x}}
\def\y{{\mathbf y}}

\def\w{{\mathbf w}}
\def\p{{\mathbf p}}

\newcommand{\indep}{\!\perp\!\!\!\perp}

\usepackage{tikz}
\usetikzlibrary{arrows}
\tikzset{
state/.style = {shape=circle,draw,thick,minimum size=3.0em,font=\huge},
sstate/.style = {draw,thick,minimum size=3.0em,font=\huge},
dstate/.style = {shape=circle,draw,thick,double,minimum size=3.0em,font=\huge},
square/.style={regular polygon,regular polygon sides=4},
bidirected/.style={Latex-Latex,dashed},
}
\usetikzlibrary{shapes,decorations,calc,fit,positioning}
\usetikzlibrary{arrows.meta}
\tikzset{>={Latex[width=1.5mm,length=1.5mm]}}

\newcommand\dagsize{0.35}

\title{Constrained Identifiability of Causal Effects}
\author{
    Yizuo Chen\textsuperscript{\rm 1},
    Adnan Darwiche\textsuperscript{\rm 1}
}
\affiliations{

    \textsuperscript{\rm 1}Computer Science Department,
    University of California, Los Angeles, USA\\
    yizuo.chen@ucla.edu\ \ darwiche@cs.ucla.edu\\
}

\usepackage{bibentry}

\begin{document}

\maketitle

\begin{abstract}
We study the identification of causal effects in the presence of different types of constraints (e.g., logical constraints) in addition to the causal graph. These constraints impose restrictions on the models (parameterizations) induced by the causal graph, reducing the set of models considered by the identifiability problem. We formalize the notion of \emph{constrained identifiability}, which takes a set of constraints as another input to the classical definition of identifiability. We then introduce a framework for testing constrained identifiability by employing tractable Arithmetic Circuits (ACs), which enables us to accommodate constraints systematically. We show that this AC-based approach is at least as complete as existing algorithms (e.g., do-calculus) for testing classical identifiability, which only assumes the constraint of strict positivity. We use examples to demonstrate the effectiveness of this AC-based approach by showing that unidentifiable causal effects may become identifiable under different types of constraints.
\end{abstract}

\section{Introduction}
A causal effect measures the impact of an intervention on an outcome of interest. For example, one may ask the question ``how likely would an employee resign if the company reduced the bonus?'' More generally, given a set of treatment variables \(\X\) and another set of outcome variables \(\Y,\) the causal effect is defined as the probability of \(\y\) under a treatment \(do(\x)\) and is commonly denoted by \(\Pr(\y|do(\x))\) or \(\Pr_\x(\y).\) These queries belong to the second rung of Pearl's causal hierarchy and cannot be answered in general without conducting experiments~\citep{pearl18}. When a causal graph is available, however, some of these causal effects may be answered from observational studies. This leads to the problem of causal-effect identifiability that studies whether a causal effect can be uniquely determined from observational distributions when a causal graph is given; see, e.g.,~\citep{pearl00b,imbens_rubin_2015,PetersBook,SpirtesBook}.

A recent line of research went beyond classical identifiability to show that more causal effects may become identifiable when additional information besides a causal graph is available. Works that fall into this line include the exploitation of the knowledge of context-specific independence~\citep{uai/BoutilierFGK96,nips/TikkaHK19} and more recently, functional dependencies~\citep{neurips/ChenDarwiche24}. In this work, we show that such knowledge can be understood as \emph{constraints} on the models (parameterizations) induced by the causal graph. This leads us to formulate the notion of \emph{constrained identifiability},\footnote{The notion was initially introduced in~\citep{neurips/ChenDarwiche24} to address positivity assumptions (constraints). We generalize the notion here to incorporate any constraints.} which takes a set of models (defined by constraints) as an additional input to the classical definition of identifiability. Constrained identifiability is general enough to incorporate more types of constraints, such as fully-known observational distributions, as shown in this work. Since constraints reduce the distributions considered by identifiability, an unidentifiable causal effect may become identifiable due to available constraints. 
\begin{figure}[tb]
\centering
\begin{subfigure}[r]{0.40\linewidth}
\centering
\begin{tikzpicture}[->=stealth,auto,scale=0.45,transform shape]
\node[font=\huge] (A) at (0,0) {$A$};
\node[font=\huge] (B) at (3,0) {$B$};
\node[font=\huge] (theta1A) at (-1,2) {$\theta^1_{A} = [0.2,0.8]$};
\node[font=\huge] (theta2A) at (-1,1) {$\theta^2_{A} = [0.3,0.7]$};
\node[font=\huge] (theta1B1) at (3,4) {$\theta^1_{B|a} = [0.1,0.9]$};
\node[font=\huge] (theta1B2) at (3,3) {$\theta^1_{B|\bar{a}} = [0.4,0.6]$};
\node[font=\huge] (theta2B1) at (3,2) {$\theta^2_{B|a} = [0.5,0.5]$};
\node[font=\huge] (theta2B2) at (3,1) {$\theta^2_{B|\bar{a}} = [1,0]$};
\path (A) edge (B);
\end{tikzpicture}
\caption{causal graph}
\label{sfig:intro-1}
\end{subfigure}
\begin{subfigure}[r]{0.29\linewidth}
\centering
\includegraphics[width=\linewidth]{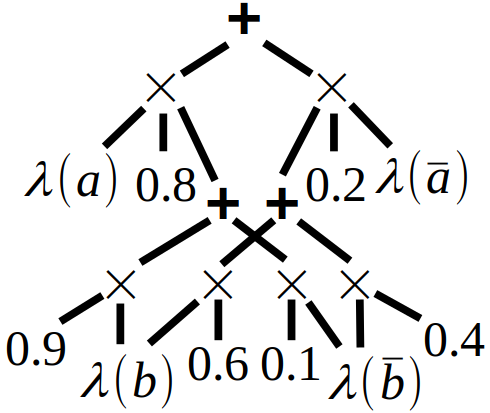}
\caption{$AC_1$}
\label{sfig:intro-2}
\end{subfigure}
\begin{subfigure}[r]{0.29\linewidth}
\centering
\includegraphics[width=\linewidth]{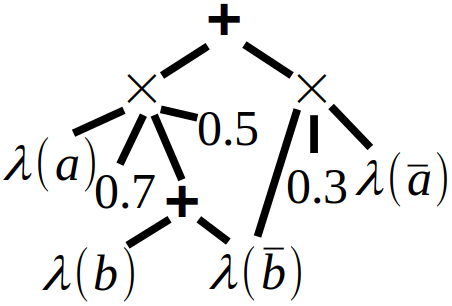}
\caption{$AC_2$}
\label{sfig:intro-3}
\end{subfigure}
\caption{Arithmetic circuits for two different models.}
\label{fig:intro}
\end{figure}
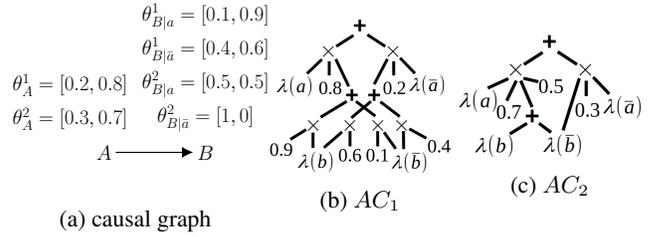

We further propose an approach for testing constrained identifiability based on Arithmetic Circuits (AC)~\citep{darwiche/dnnf,darwiche2003,DarwicheBook09,faia/Darwiche21}. In essence, an AC is a data structure representing a particular computation procedure and has played an influential role in probabilistic inference. In this work, we show yet another application of ACs in testing constrained identifiability based on the following observation: A causal effect \(\Pr_\x(\y)\) is identifiable iff there exists an AC that computes \(\Pr_\x(\y)\) and whose output only depends on the observational distribution \(\Pr(\V).\) That is, if we can construct an AC to compute the causal effect and certify that the output of the AC is invariant under any fixed \(\Pr(\V),\) the causal effect is guaranteed to be identifiable. The construction of the AC can be done using existing knowledge compilation methods, which enables us to exploit parameter-specific constraints of the model such as equal parameters (e.g., context-specific independences) and 0/1 parameters (e.g., logical constraints) and yield simpler ACs; see, e.g.,~\citep{kr/Darwiche02,Chavira.Darwiche.Ijcai.2007,corr/counter-ACE,flairs/0002D24}. Consider two different parameterizations \(\Theta^1, \Theta^2\) for the same causal graph $G$ in Figure~\ref{sfig:intro-1} which contains two binary variables \(A, B.\) The AC for \((G,\Theta^2)\) in Figure~\ref{sfig:intro-3} contains fewer nodes than the AC for \((G,\Theta^1)\) in Figure~\ref{sfig:intro-2} since \(\Theta^2\) exhibits more constraints that get exploited when constructing the AC. As we will see later, constructing simpler ACs may lead to more complete causal-effect identifications using the method we shall propose. To check whether the output of an AC is invariant, we propose a method based on deriving an expression that is equivalent to the AC output and only involves the observed variables \(\V.\) We will show that this AC-based method can treat constraints in the form of context-specific independencies, functional dependencies, and fully-known observational distributions (that is, when $\Pr(\V)$ is known).

This paper is structured as follows. We start with technical preliminaries on the problem of classical causal-effect identifiability in Section~\ref{def:background}. We discuss the role of constraints in the context of identifiability and formally define the notion of constrained identifiability in Section~\ref{sec:constrained-id}. We introduce the AC-based method for testing constrained identifiability in Section~\ref{sec:ac-method} and propose a method for testing whether the output of an AC is invariant in Section~\ref{sec:projection}. We demonstrate the effectiveness of AC-based method with examples in Section~\ref{sec:examples}. We close with some concluding remarks in Section~\ref{sec:conclusion}. All the proofs are included in the Appendix.

\section{Technical Preliminaries}
\label{def:background}
We consider discrete variables in this work.
Single variables are denoted by uppercase letters (e.g., \(X\))
and their states are denoted by lowercase letters (e.g., \(x\)).
Sets of variables are denoted by bold uppercase letters (e.g., \(\X\))
and their instantiations are denoted by bold lowercase letters (e.g., \(\x\)).

The problem of causal-effect identifiability studies whether a causal effect can be uniquely computed from a given causal graph \(G\) and a subset of observed variables \(\V\). That is, to show that a causal effect \(\Pr_\x(\y)\) is identifiable, we need to prove that all parameterizations for \(G\) that induce the same observational distribution \(\Pr(\V)\) also yield the same value for \(\Pr_\x(\y).\) While the general definition of identifiability (not necessarily for causal effects) in~\citep{pearl00b} does not restrict the observational distribution \(\Pr(\V),\)
positivity assumptions (constraints) are often assumed to prevent zeros in \(\Pr(\V)\); see~\citep{neurips/ChenDarwiche24,uai/KivvaMEK22,icml/HwangCKL24} for recent discussions on positivity. The most common form of positivity constraints in the existing literature is strict positivity, i.e., \(\Pr(\V) > 0\), which is a sufficient condition for complete causal-effect identification algorithms including the ID algorithm~\citep{aaai/ShpitserP06} and the do-calculus~\citep{pearl00b}.\footnote{\citep{neurips/ChenDarwiche24} showed that most causal effects are not identifiable unless we assume the positivity constraint \(\Pr(X) > 0\) for each treatment variable \(X.\)} We will refer to identifiability with strict positivity as ``classical identifiability,'' which is defined next. 

\begin{definition}[Causal-Effect Identifiability]
\label{def:identifiability}
Let \(G\) be a casual graph and \(\V\) be its observed variables. A causal effect \(\Pr_\x(\y)\) is said to be \underline{identifiable} with respect to $\langle G$, $\V \rangle$ if \(\Pr^1_{\x}(\y)=\) \(\Pr^2_{\x}(\y)\)
for any pair of models \(M^1, M^2\) that induce \(\Pr^1, \Pr^2\) where \(\Pr^1(\V)=\) \(\Pr^2(\V) > 0.\)
\end{definition}

A causal effect is \emph{unidentifiable} if it is not identifiable. To show that a causal effect is unidentifiable, it suffices to find two parametrizations for \(G\) that induce the same distribution \(\Pr(\V) > 0\) yet different values for the causal effect \(\Pr_\x(\y)\).

\section{Constrained Identifiability}
\label{sec:constrained-id}
We next generalize this setup by considering more information, in addition to the causal graph and observed variables, as inputs to the identifiability problem. Past works that fit in this direction include the exploitation of \emph{context-specific independences} in~\citep{nips/TikkaHK19} and \emph{functional dependencies} in~\citep{neurips/ChenDarwiche24} to improve the identifiability of causal effects. The general setup we formulate in this work will allow us to also consider further information types such as \emph{fully-known observational distributions \(\Pr(\V),\)} which can be readily available in practice when data on \(\V\) allows us to estimate \(\Pr(\V).\)

Our setup is based on the notion of \emph{constrained identifiability,} which was recently introduced in~\citep{neurips/ChenDarwiche24} to systematically treat a variety of positivity constraints (assumptions). We will extend this notion next to incorporate arbitrary types of constraints.

We first define the (constrained) identifiability tuple. Here, we use ``model'' to mean a full parametrization of the causal graph (a model induces a distribution).
\begin{definition}
\label{def:constrained-tuple}
We call \(\langle G, \V, \MM \rangle\) an \underline{identifiability tuple} when \(G\) is a causal graph (DAG), \(\V\) is its set of observed variables, and \(\MM\) is a set of  models induced by \(G.\)
\end{definition}
The notion of constrained-identifiability can now be defined as follows. For simplicity, we will say ``identifiability'' to mean ``constrained-identifiability'' in the rest of paper.
\begin{definition}
\label{def:constrained-id}
Let \(\langle G, \V, \MM \rangle\) be an identifiability tuple. A causal effect \(\Pr_\x(\y)\) is said to be \underline{identifiable} wrt \(\langle G, \V, \MM \rangle\) if
\(\Pr^1_{\x}(\y)=\) \(\Pr^2_{\x}(\y)\)
for any pair of models \(M^1, M^2 \in \MM\) that induce \(\Pr^1, \Pr^2\) where \(\Pr^1(\V)=\) \(\Pr^2(\V).\)
\end{definition}

Positivity constraints, context-specific independencies (CSI), and functional independencies restrict the set of models \(\MM\) in the definition of constrained-identifiability. That is, we only consider models that induce a distribution \(\Pr\) that satisfies all the constraints above when testing for identifiability. To illustrate, consider the causal graph in Figure~\ref{sfig:strong-uid1} where all variables are binary and \(X, Y\) are observed. If we impose the positivity constraint \(\Pr(X, Y) > 0,\) CSI constraint (\(Y \indep U | x\)), and functional variable \(Y,\) then a model will be excluded from \(\MM\) if it assigns \(Y \gets X\) or \(Y \gets X \oplus U\) as a structural equation for \(Y.\)\footnote{The CSI constraint (\(Y \indep U | x\)) says that \(Y\) is independent of \(U\) when \(X=x.\)  Variable \(Y\) is functional means that \(Y\) is functionally determined by its parents (\(U\) and \(X\)). We will formally define these constraints in Section~\ref{sec:examples}.}

If \(\CC\) is a set of constraints, we will use \(\MM(\CC)\) to denote the set of models that satisfy \(\CC.\) In the example above, we write the set of models as \(\MM[\Pr(X, Y) > 0, (Y \indep U | x), (\W = \{Y\})],\) where $\W$ denotes variables with functional dependencies. When no constraint is imposed, we attain the weakest version of identifiability with \(\MM(\emptyset).\) Classical identifiability (Definition~\ref{def:identifiability}) corresponds to \(\langle G, \V, \MM[\Pr(\V) > 0] \rangle.\) 

One type of constrained identifiability arises when we restrict the observational distribution to a fixed \(\Pr^{\star}(\V).\) That is, we only consider models that induce the given \(\Pr^{\star}(\V).\) 
To illustrate, consider the diagram in Figure~\ref{sfig:d-id-diag}, which contains the set of all models \(\MM\) induced by a causal graph. We can partition the set of models into \textit{equivalence classes,} where all models in a particular class induce the same distribution \(\Pr(\V)\) --- hence, an equivalence class is defined by its distribution $\Pr(\V).$\footnote{There is an infinite number of such equivalence classes since
a causal graph $G$ induces infinitely many distributions \(\Pr(\V).\)}
In classical identifiability, to show that a causal effect is identifiable, we need to show that all models in the same equivalence class produce the same causal effect \(\Pr_\x(\y)\). When \(\Pr^{\star}(\V)\) is fixed, however, we only consider a single equivalence class \(\Pr^{\star}(\V)\) and check whether all models in this particular class produce the same value for the causal effect \(\Pr_\x(\y)\). We denote the identifiability tuple under a fixed \(\Pr^{\star}(\V)\) as \(\langle G, \V, \MM[\Pr(\V)=\Pr^{\star}(\V)]\rangle.\)

\begin{figure}[tb]
\vspace{-2.5em}
\centering
\includegraphics[width=0.35\linewidth]{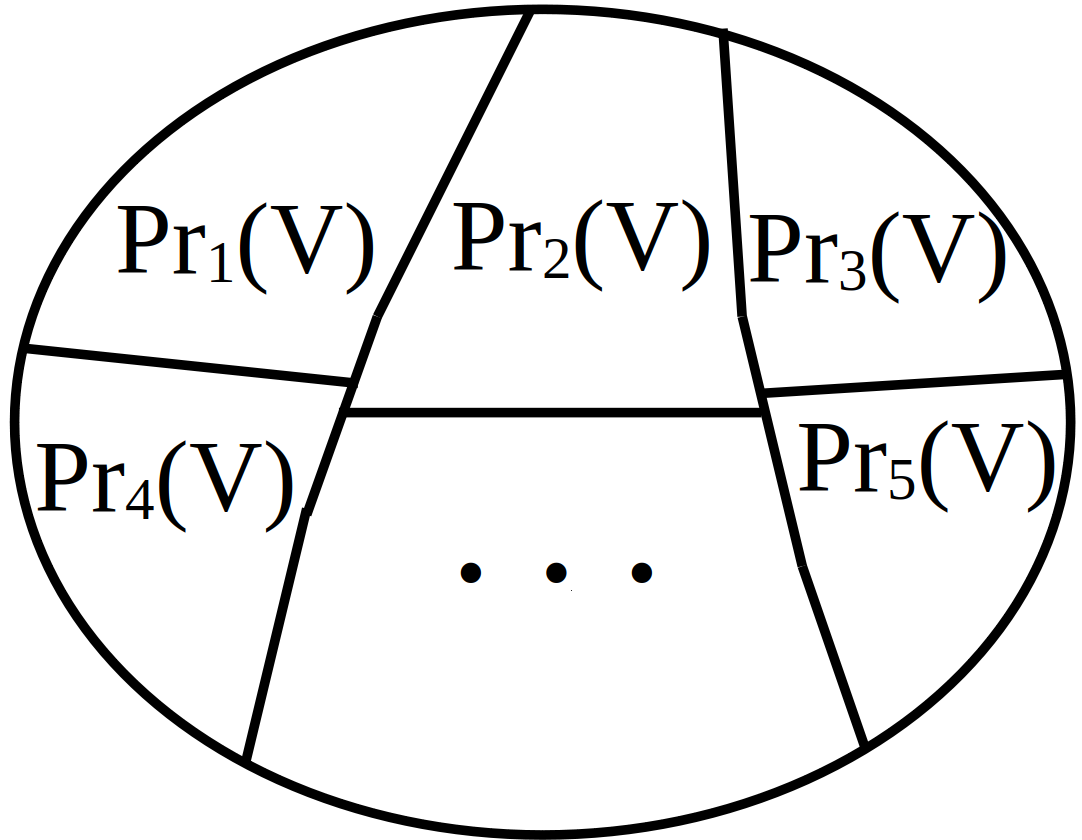}
\caption{models for \(\Pr(\V)\)}
\label{sfig:d-id-diag}
\end{figure}

A causal effect is identifiable by Definition~\ref{def:identifiability} iff it is identifiable under \emph{all} fixed \(\Pr^\star(\V) > 0.\) However, a causal effect that is not identifiable by Definition~\ref{def:identifiability} may still be identifiable under \emph{some} \(\Pr^\star(\V) > 0.\) This happens when we have two distinct equivalence classes $C_1$ and $C_2$ where the models in $C_1$ agree on the causal effect but those in $C_2$ disagree. We next show a causal effect that belongs to this category.\footnote{The proof of Proposition~\ref{prop:weak-weak-1} exploits the context-specific independence relations \((Y \indep C\ |\ A=a, B)\) and \((Y \indep B\ |\ A=\bar{a}, C)\); see~\citep{nips/TikkaHK19} for more details about leveraging this type of constraints for identifiability.}

\begin{proposition}
\label{prop:weak-weak-1}
Consider the causal graph \(G\) in Figure~\ref{sfig:weak-weak1} where \(\V = \{X,\) \(A,\) \(B,\) \(C,\) \(Y\}.\) The causal effect \(\Pr_x(y)\) is unidentifiable wrt \(\langle G, \V, \MM[\Pr(\V)>0]\rangle\) but is identifiable wrt \(\langle G, \V, \MM[\Pr(\V) = \Pr^\star(\V)]\rangle\) for some \(\Pr^\star(\V) > 0.\)
\end{proposition}

The following causal effect is not identifiable for all \(\Pr^\star(\V) > 0\) (it is not identifiable according to Definition~\ref{def:identifiability}).

\begin{proposition}
\label{prop:strong-uid1}
Consider the causal graph \(G\) in Figure~\ref{sfig:strong-uid1} where \(\V = \{X, Y\}\). The causal effect \(\Pr_x(y)\) is unidentifiable wrt \(\langle G, \V, \MM[\Pr(\V) = \Pr^\star(\V)] \rangle\) for all \(\Pr^\star(\V) > 0.\) 
\end{proposition}

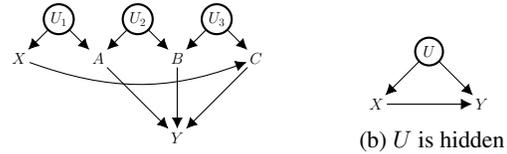
\begin{figure}[tb]
\centering
\begin{subfigure}[r]{0.45\linewidth}
\centering
\begin{tikzpicture}[->=stealth,auto,scale=\dagsize,transform shape]
\node[font=\huge] (X) at (0,0) {$X$};
\node[state,font=\huge] (U1) at (1.5,1.5) {$U_1$};
\node[font=\huge] (A) at (3,0) {$A$};
\node[state,font=\huge] (U2) at (4.5,1.5) {$U_2$};
\node[font=\huge] (B) at (6,0) {$B$};
\node[state,font=\huge] (U3) at (7.5,1.5) {$U_3$};
\node[font=\huge] (C) at (9,0) {$C$};
\node[font=\huge] (Y) at (6,-3) {$Y$};

\path (U1) edge (X);
\path (U1) edge (A);
\path (U2) edge (A);
\path (U2) edge (B);
\path (U3) edge (B);
\path (U3) edge (C);
\path (X) edge[bend right=20] (C);
\path (A) edge (Y);
\path (B) edge (Y);
\path (C) edge (Y);
\end{tikzpicture}
\caption{\(U_1,U_2,U_3\) are hidden}
\label{sfig:weak-weak1}
\end{subfigure}
\begin{subfigure}[r]{0.45\linewidth}
\centering
\begin{tikzpicture}[->=stealth,auto,scale=\dagsize,transform shape]
\node[font=\huge] (X) at (0,0) {$X$};
\node[font=\huge] (Y) at (4,0) {$Y$};
\node[state,font=\huge] (U) at (2,2) {$U$};

\path (U) edge (X);
\path (U) edge (Y);
\path (X) edge (Y);
\end{tikzpicture}
\caption{\(U\) is hidden}
\label{sfig:strong-uid1}
\end{subfigure}
\caption{causal graphs that exhibit different behaviors of identifiability for \(\Pr_x(y).\)}
\label{}
\end{figure}

\section{Testing Constrained-Identifiability Using Arithmetic Circuits}
\label{sec:ac-method}
We next introduce a general method for testing constrained-identifiability based on Arithmetic Circuits (ACs) as defined in~\citep{darwiche/dnnf,darwiche2003,DarwicheBook09}. One main advantage of using ACs for testing identifiability is that more local structures such as the 0/1 parameters and equal parameters (including context-specific independences) can be exploited by the existing knowledge compilation methods~\citep{kr/Darwiche02,Chavira.Darwiche.Ijcai.2005,Chavira.Darwiche.Ijcai.2007}, leading to more complete causal-effect identifications under the additional constraints. 
Throughout the paper we will assume that the states of all variables (including \(\U\)) are known, and we will point out in case this assumption can be omitted.\footnote{\citep{icml/Zhang0B22} showed that all counterfactual distributions (which subsumes observational and interventional distributions) can be captured by causal models whose hidden variables have a \emph{bounded} number of states.}

Our method for testing identifiability using ACs involves two steps (ingredients): \emph{AC Construction} and \emph{Invariance Testing}. The first step constructs an AC whose output computes the causal effect \(\Pr_\x(\y)\) and which incorporates the constraints. The second step tests identifiability by checking whether the output of the AC is invariant under all models that satisfy the constraints (\(\MM\) in Definition~\ref{def:constrained-id}). Hence, the main objective of this and future work is to improve the quality of these two steps so that we obtain a more complete method for identifying causal effects under constraints.

\subsection{AC Construction}
We start by constructing an AC for the causal effect \(\Pr_\x(\y)\) based on the approach in~\citep{causalityAC}. First, compile the original BN into an AC (in a standard way) that contains indicators \(\lambda_{v}\) for each variable state \(v\) and parameters \(\theta_{v|\p}\) for each instantiation over \(V\) and its parents \(\P.\)\footnote{Note that each \(\theta_{v|\p}\) is equivalent to the conditional probability \(\Pr(v|\p)\) when \(\Pr(\p) > 0.\) Hence, we will use \(\theta_{v|\p}\) and \(\Pr(v|\p)\) interchangeably when the positivity condition holds.} Second, replace every parameter \(\theta_{x | \p}\) with 1 for each treatment \(X \in \X.\) Finally, for each \(V \in \X \cup \Y\), assign \(\lambda_v = 1\) if \(v\) is compatible with \(\x,\y\) and \(\lambda_v = 0\) otherwise; assign 1 to all other indicators.\footnote{The assignment is identical to the one used for AC estimation under evidence \(\x, \y.\)}

To evaluate an AC under a model \(M,\) we simply plugin each (non-constant) parameter \(\theta_{v | \p}\) with the value from \(M\) and compute the output of the AC in a standard way.
One key observation is that the ACs constructed from the above procedure are guaranteed to compute \(\Pr_\x(\y)\) correctly under \emph{all} parameterizations of \(G.\) As we will see later, when additional constraints are available, the ACs may be further simplified, leading to more complete causal-effect identifiability results in practice. In general, the leaves of an AC are either constants (e.g., 0.8) or parameters (\(\theta_{v | \p}\)). 

Consider the causal graph \(G\) in Figure~\ref{sfig:id-ac7-1} with observed variables \(\V = \{X, A, B, Y\}.\) Assuming all variables are binary, Figure~\ref{sfig:id-ac7-2} depicts an AC for \(\Pr_x(y).\) Due to the space limit, we only plotted one branch for most \(+\)-nodes in the figure. In general, we prefer simpler (smaller) ACs since they can sometimes lead to more complete causal-effect identifiability. While most AC simplifications can be handled automatically by the \textsc{ace} package,\footnote{http://reasoning.cs.ucla.edu/ace.} we propose additional rules that can be used to further simplify ACs in certain scenarios; see Appendix~\ref{app:simplify-AC} for more details.

\begin{figure}[tb]
\centering
\begin{subfigure}[r]{0.28\linewidth}
\centering
\begin{tikzpicture}[->=stealth,auto,scale=\dagsize,transform shape]
\node[font=\huge] (X) at (0,0) {$X$};
\node[state,font=\huge] (U1) at (1.5,1.5) {$U_1$};
\node[font=\huge] (A) at (3,0) {$A$};
\node[state,font=\huge] (U2) at (4.5,1.5) {$U_2$};
\node[font=\huge] (B) at (6,0) {$B$};
\node[font=\huge] (Y) at (3,-3) {$Y$};

\path (U1) edge (X);
\path (U1) edge (A);
\path (U2) edge (A);
\path (U2) edge (B);
\path (X) edge (Y);
\path (A) edge (Y);
\path (B) edge (Y);
\end{tikzpicture}
\caption{causal graph}
\label{sfig:id-ac7-1}
\end{subfigure}
\begin{subfigure}[r]{0.35\linewidth}
\centering
\includegraphics[width=\linewidth]{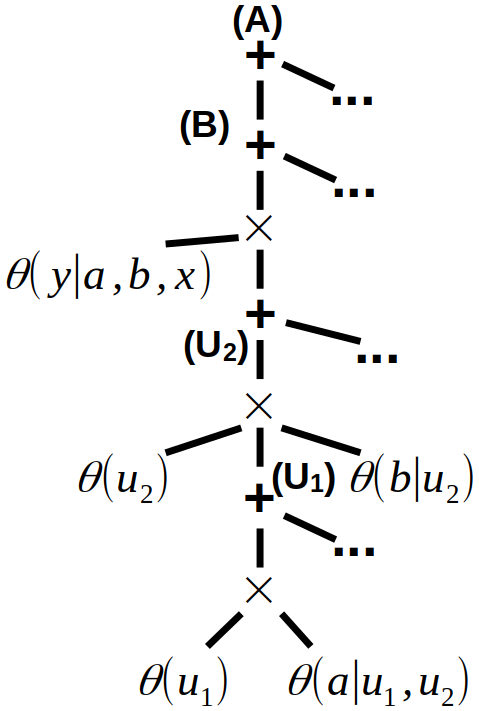}
\caption{AC}
\label{sfig:id-ac7-2}
\end{subfigure}
\begin{subfigure}[r]{0.35\linewidth}
\centering
\includegraphics[width=\linewidth]{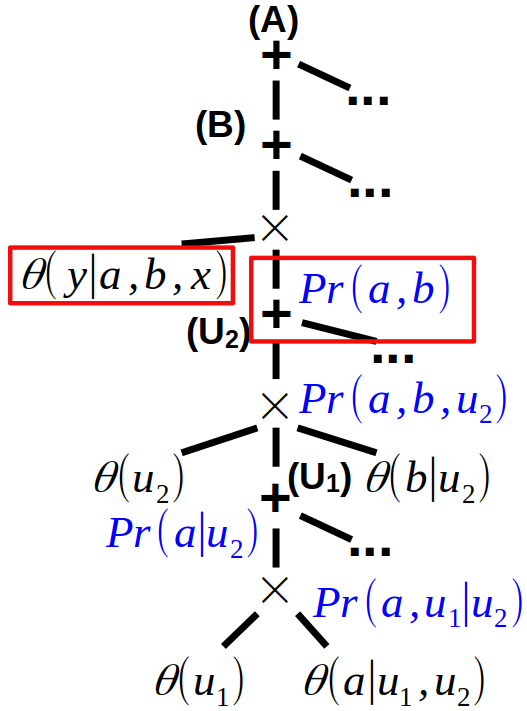}
\caption{Invariant-cut}
\label{sfig:id-ac7-3}
\end{subfigure}
\caption{AC constructed for \(\Pr_x(y)\) for a causal graph; only one branch for each \(+\)-node is plotted; equivalent expressions are shown in blue and nodes in the invariant-cut are marked with red boxes.}
\label{fig:ac-simplified}
\end{figure}
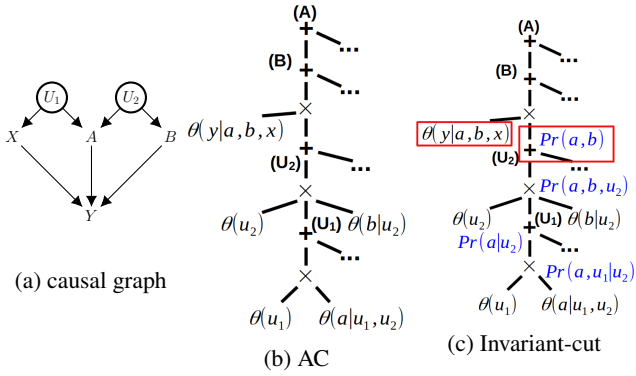

\subsection{Invariance Testing}
The AC constructed from the previous step is guaranteed to compute the causal effect under all models in \(\MM.\) We next develop a method for testing identifiability based on the following observation: a causal effect \(\Pr_\x(\y)\) is identifiable wrt \(\langle G, \V, \MM \rangle\) if the output of the ACs for \(\Pr_\x(\y)\) is \emph{invariant} under all models (in \(\MM\)) that induce a same \(\Pr(\V).\)

In this work, an \emph{expression} is defined as a composition of constants (e.g., 0.8), parameters (\(\theta_{v|\p}\)), and conditional probabilities (e.g., \(\Pr(a | b)\)). Note that each parameter \(\theta_{v|\p}\) is equal to the conditional probability \(\Pr(v | \p)\) when \(\Pr(\p) > 0.\) For example, \(\Pr(y | x),\) \(0.65,\) \(0.2\Pr(a|x) + 0.6\Pr(x),\) \(\sum_u \Pr(a|x)\Pr(u),\) \(\theta(a|x) \Pr(x) + 0.3\Pr(\bar{a})\) are all valid expressions. An AC specifies an expression consisting of \(\times\)'s and \(+\)'s, which can be obtained by evaluating the output of the AC in a standard (bottom-up) way. For example, the AC in Figure~\ref{sfig:id-ac7-2} specifies the expression \(\sum_a \sum_b \theta_{y | a,b,x} \sum_{u_2} \theta_{u_2}\theta_{b|u_2}\sum_{u_1}\theta_{u_1}\theta_{a|u_1,u_2}.\) 
Moreover, every node in an AC corresponds to an expression so we will use ``AC node'' and ``expression'' interchangeably.

Our next goal is to propose a method for testing identifiability that operates on ACs. We start by defining some key notions that are required to achieve this goal. 

\begin{definition}
\label{def:inv-exp}
An expression is called \underline{\(\langle G, \V, \MM \rangle\)-invariant} if it evaluates to the same value for all models in \(\MM\) that induce a same \(\Pr(\V).\)
\end{definition}

The following result shows that we can test identifiability by checking whether the output of an AC is invariant.

\begin{proposition}
\label{prop:did}
A causal effect \(\Pr_\x(\y)\) is identifiable wrt \(\langle G, \V, \MM \rangle\) iff there is an AC for \(\Pr_\x(\y)\) whose output expression is \(\langle G, \V, \MM \rangle\)-invariant.
\end{proposition}

In general, it is difficult to tell whether the output of an AC is invariant since the evaluation of the AC often contains hidden variables \(\U.\) However, if we are able to replace the expressions for some AC nodes with \emph{equivalent expressions} that do not depend on \(\U,\) we may obtain output expressions that are obviously invariant, e.g., expressions that only involve observed variables \(\V.\)

\begin{definition}
\label{def:eq-exp}
Two expressions \(\EE_1\) and \(\EE_2\) are called \underline{\(\langle G, \V, \MM \rangle\)-equivalent} if they evaluate to the same value under all models \(\MM.\) If in addition \(\EE_1\) does not contain any parameters (\(\theta\)'s)\footnote{The observational distribution \(\Pr(\V)\) does not depend on the value of parameter \(\theta_{x|\p}\) 
when \(\Pr(\p) = 0\) ---
but the interventional distribution $\Pr_\x(.)$ and, hence, causal effects 
may depend on the value of \(\theta_{x|\p}\) in this case. 
This is why we require projections to be free of parameters as they may not be computable from $\Pr(\V)$. However, since \(\theta_{x|\p}=\Pr(x|\p)\) when \(\Pr(\p) > 0\) we can always replace \(\theta_{x|\p}\) with \(\Pr(x|\p)\) when \(\Pr(\p) > 0.\)} and does not mention any hidden variables, then \(\EE_1\) is called an \underline{\(\langle G, \V, \MM \rangle\)-projection} of \(\EE_2.\)
\end{definition}

The following theorem presents a method for testing identifiability.

\begin{proposition}
\label{prop:id-cut2}
The causal effect \(\Pr_\x(\y)\) is identifiable wrt \(\langle G, \V, \MM \rangle\) iff there is an AC for \(\Pr_\x(\y)\) that satisfies the following condition: there is a cut between the root and leaf nodes where every node on the cut has a \(\langle G, \V, \MM \rangle\)-projection.
\end{proposition}

We call the cut in Proposition~\ref{prop:id-cut2} an \emph{invariant-cut}. Figure~\ref{sfig:id-ac7-3} depicts the projections (in blue) and an invariant-cut (in red boxes) for the AC in Figure~\ref{sfig:id-ac7-2}. Hence, we conclude that the causal effect \(\Pr_x(y)\) is identifiable. Once we attain an invariant-cut we immediately obtain an identifying formula by plugging in the projections for the nodes on the invariant-cut and evaluating the AC root.  In this example, the identifying formula for the causal effect is \(\Pr_x(y) = \Pr(y | x, a, b)\)\(\Pr(a,b)\) \(+\) \(\Pr(y | x, a, \bar{b})\)\(\Pr(a,\bar{b})\) \(+\) \(\Pr(y | x, \bar{a}, b)\)\(\Pr(\bar{a},b)\) \(+\) \(\Pr(y | x, \bar{a}, \bar{b})\)\(\Pr(\bar{a},\bar{b}),\) which resembles the backdoor adjustment formula. Invariant-cuts may not be unique as the choice of projections is not unique. We may also have multiple ACs with invariant-cuts, which further contributes to the multiplicity of identifying formulas.

Another interesting observation on Proposition~\ref{prop:id-cut2} is that the root of the AC always constitutes a valid invariant-cut if the causal effect is identifiable. Consider the following expression (ratio of two infinite series) for the root: \(\frac{\sum_{M \in \MM} \mathbbm{1}\{\Pr(\V) = \Pr^M(\V)\} \Pr^M_{\x}(\y)}{\sum_{M \in \MM} \mathbbm{1}\{\Pr(\V) = \Pr^M(\V)\}}\) where \(\Pr^M \) denotes the distribution induced by a specific model \(M.\) If \(\Pr_\x(\y)\) is identifiable, the expression is a valid identifying formula for computing \(\Pr_\x(\y)\) since every model \(M\) that induces \(\Pr(\V)\) also induces a same \(\Pr^M_{\x}(\y).\) Hence, the root of the AC always constitutes an invariant-cut in this case since any identifying formula constitutes a projection for the root. However, since lower cuts involve smaller
expressions, finding efficient projections for nodes on such cuts may be easier.

\section{Finding Invariant-Cuts For ACs}
\label{sec:projection}

Developing a complete method for finding invariant-cuts, if they exist, is beyond the scope of this paper as it represents a first step in the proposed direction. However, we will provide a complete method in this section assuming we only have the classical positivity constraint, $\Pr(\V)>0.$ This case has already been solved in the literature using methods such as the ID algorithm and the do-calculus, so 
the results in this section provide another way of solving this case. These results also show that the proposed framework in this paper is at least as powerful as these classical methods.
Moreover, we will use these results in the next section where we show how to handle other types of constraints but in a less complete fashion as far as identifying invariant-cuts.

Our method for finding projections for ACs is based on the notion of c-components, which has been studied extensively in~\citep{uai/TianP02b,TianPearl03,aaai/ShpitserP06,aaai/HuangV06}. Given a causal graph \(G\), a (maximal) c-component is a (maximal) subgraph of \(G\) in which all variables are connected by bidirected paths,\footnote{A bidirected path is a path that consists of bidirected edges in the form of \(X \leftarrow U \rightarrow Y.\) We assume the causal graph \(G\) is semi-Markovian here; that is, all hidden variables are roots in \(G\) and have exactly two children. A more general definition of c-components beyond semi-Markovian can be found in~\citep{uai/TianP02b}.} and each c-component \(S\) is associated with a \emph{c-factor} defined as \(Q[S](\V) = \sum_\U \prod_{U \in \U} \theta_U \prod_{V \in \V} \theta_{V | \P_V}\) where \(\U, \V, \P_V\) denote the hidden variables, observed variables, and parents of \(V.\) Moreover, the causal graph \(G\) can be decomposed into c-components \(\SS = \{S_1, \dots, S_k\}\) such that \(\Pr(\V) = \prod_{S \in \SS} Q[S]\) when \(\Pr(\V) > 0.\) Existing methods in~\citep{TianPearl03,aaai/ShpitserP06,aaai/HuangV06} identify causal effects by decomposing the mutilated graph \(G'\) (under \(do(\x)\)) into c-components \(\SS\) and then checking whether each \(Q[S]\) (for \(S \in \SS\)) has a projection which only involves \(\V.\)\footnote{A mutilated graph under \(do(\x)\) is attained by removing all the incoming edges for \(\X\) from the original causal graph.} 
We next present a notion that summarizes such c-components, which follows from the key results in~\citep{TianPearl03,aaai/HuangV06,aaai/ShpitserP06} and is defined in an inductive way.

\begin{definition}
\label{def:id-ccomp}
Let \(G\) be a causal graph with observed variables \(\V.\) A subgraph \(S\) of \(G\) is \underline{\(\V\)-computable} if one of the following holds: (1) \(S = G\); (2) \(S\) is a maximal c-component of a \(\V\)-computable subgraph; or  (3) \(S\) is the result of pruning a leaf node from a \(\V\)-computable subgraph.
\end{definition}

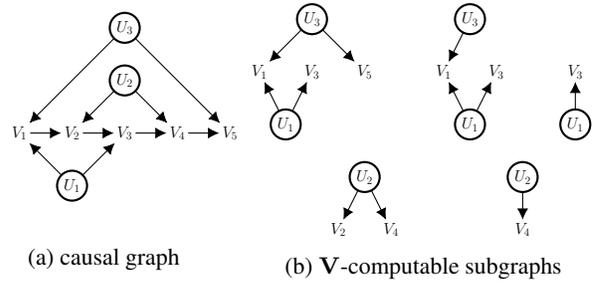
\begin{figure}[tb]
\centering
\begin{subfigure}[r]{0.3\linewidth}
\centering
\begin{tikzpicture}[->=stealth,auto,scale=\dagsize,transform shape]
\node[font=\huge] (V1) at (0,0) {$V_1$};
\node[font=\huge] (V2) at (2,0) {$V_2$};
\node[font=\huge] (V3) at (4,0) {$V_3$};
\node[font=\huge] (V4) at (6,0) {$V_4$};
\node[font=\huge] (V5) at (8,0) {$V_5$};
\node[state,font=\huge] (U1) at (2,-2) {$U_1$};
\node[state,font=\huge] (U2) at (4,2) {$U_2$};
\node[state,font=\huge] (U3) at (4,4) {$U_3$};

\path (V1) edge (V2);
\path (V2) edge (V3);
\path (V3) edge (V4);
\path (V4) edge (V5);
\path (U1) edge (V1);
\path (U1) edge (V3);
\path (U2) edge (V2);
\path (U2) edge (V4);
\path (U3) edge (V1);
\path (U3) edge (V5);
\end{tikzpicture}
\caption{causal graph}
\label{sfig:c-components1}
\end{subfigure}
\begin{subfigure}[r]{0.69\linewidth}
\centering
\begin{tikzpicture}[->=stealth,auto,scale=\dagsize,transform shape]
\node[font=\huge] (V11) at (0,0) {$V_1$};
\node[font=\huge] (V13) at (2,0) {$V_3$};
\node[font=\huge] (V15) at (4,0) {$V_5$};
\node[state,font=\huge] (U11) at (1,-2) {$U_1$};
\node[state,font=\huge] (U13) at (2,2) {$U_3$};

\node[font=\huge] (V21) at (7,0) {$V_1$};
\node[font=\huge] (V23) at (9,0) {$V_3$};
\node[state,font=\huge] (U21) at (8,-2) {$U_1$};
\node[state,font=\huge] (U23) at (8,2) {$U_3$};

\node[font=\huge] (V33) at (12,0) {$V_3$};
\node[state,font=\huge] (U31) at (12,-2) {$U_1$};

\node[state,font=\huge] (U42) at (4,-3) {$U_2$};
\node[font=\huge] (V42) at (3,-5) {$V_2$};
\node[font=\huge] (V44) at (5,-5) {$V_4$};

\node[state,font=\huge] (U52) at (10,-3) {$U_2$};
\node[font=\huge] (V54) at (10,-5) {$V_4$};

\path (U11) edge (V11);
\path (U11) edge (V13);
\path (U13) edge (V11);
\path (U13) edge (V15);

\path (U21) edge (V21);
\path (U21) edge (V23);
\path (U23) edge (V21);

\path (U31) edge (V33);

\path (U42) edge (V42);
\path (U42) edge (V44);

\path (U52) edge (V54);

\end{tikzpicture}
\caption{\(\V\)-computable subgraphs}
\label{sfig:c-components2}
\end{subfigure}
\caption{\(\V\)-computable subgraphs for a causal graph from~\citep{uai/TianP02b}.}
\label{fig:c-components}
\end{figure}

We can always find a projection for each \(Q[S]\) as stated by the following corollary.
\begin{corollary}
\label{corr:id-ccomp}
Let \(G\) be a causal graph with observed variables \(\V.\) If \(S\) is a \(\V\)-computable subgraph of \(G\), then \(Q[S]\) has a \(\langle G, \V, \MM[\Pr(\V) > 0] \rangle \)-projection.
\end{corollary}
To illustrate, Figure~\ref{sfig:c-components2} depicts some \(\V\)-computable subgraphs for the causal graph in Figure~\ref{sfig:c-components1}. The upper-left subgraph \(S\) in Figure~\ref{sfig:c-components2} corresponds to the following c-factor: \(Q[S] = \sum_{U_1,U_3} \theta_{U_1}\) \(\theta_{U_3}\) \(\theta_{V_1 | U_1, U_3}\) \(\theta_{V_3 | U_1, V_2}\) \(\theta_{V_5 | U_3, V_4}.\) 
Now suppose the expression at some AC node \(i\) matches the c-factor for a \(\V\)-computable subgraph, we can replace it with the corresponding projection which only involves \(\V\); see Appendix~\ref{app:c-component} for more details.

Consider the causal graph in Figure~\ref{sfig:id-circuit-ccomp2-graph} with binary observed variables \(\V = \{X, Y, A\}\) and assume \(\Pr(\V) > 0.\) Figure~\ref{sfig:id-circuit-ccomp2-circ2} depicts an AC constructed for the causal effect \(\Pr_x(y).\) If we consider the \(\V\)-computable subgraph \(S\) that contains \(U \rightarrow Y,\) we obtain the following c-factor \(Q[S] = \sum_U \theta_U\theta_{Y | U,A}\) for \(S\), which matches exactly the expression at the \(+\)-node on the left branch of the AC. Hence, we can assign the corresponding projection for the \(+\)-node as shown in Figure~\ref{sfig:id-circuit-ccomp2-circ2}. It turns out that the causal effect is identifiable since we can find an invariant-cut (in red boxed) for the AC.

\begin{figure}[tb]
\centering
\begin{subfigure}[r]{0.25\linewidth}
\centering
\begin{tikzpicture}[->=stealth,auto,scale=\dagsize,transform shape]
\node[font=\huge] (X) at (0,0) {$X$};
\node[font=\huge] (A) at (2,0) {$A$};
\node[font=\huge] (Y) at (4,0) {$Y$};
\node[state,font=\huge] (U) at (2,2) {$U$};

\path (U) edge (X);
\path (U) edge (Y);
\path (X) edge (A);
\path (A) edge (Y);
\end{tikzpicture}
\caption{causal graph}
\label{sfig:id-circuit-ccomp2-graph}
\end{subfigure}
\begin{subfigure}[r]{0.73\linewidth}
\centering
\includegraphics[width=\linewidth]{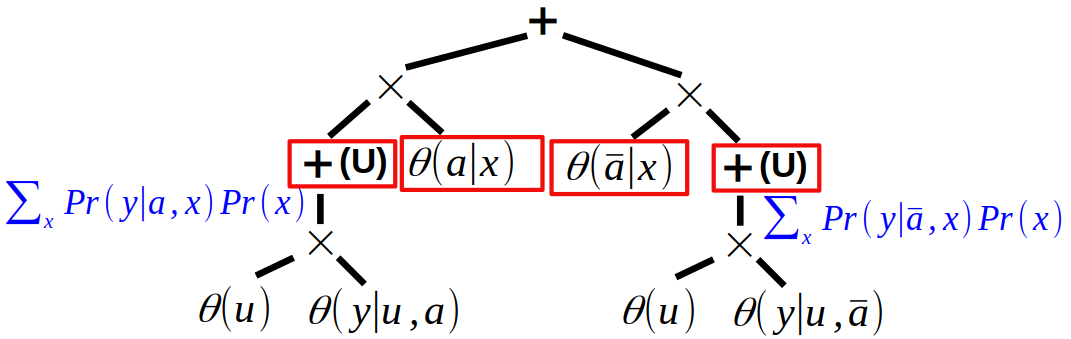}
\caption{projections and an invariant-cut}
\label{sfig:id-circuit-ccomp2-circ2}
\end{subfigure}
\caption{An AC constructed for \(\Pr_x(y)\); an invariant-cut (in red boxes) is found by the c-component method.}
\label{fig:id-circuit-ccomp2}
\end{figure}
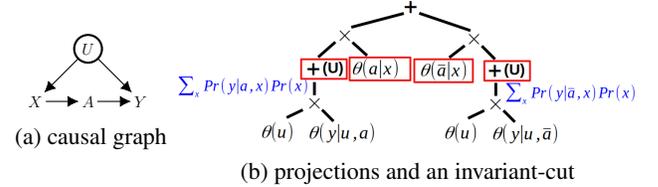

The following theorem shows that the \emph{c-component method}, when coupled with carefully designed ACs, is sound and complete for testing identifiability under $\Pr(\V)>0$ like the ID algorithm~\citep{aaai/ShpitserP06} and do-calculus~\citep{pearl00b}.

\begin{theorem}
\label{thm:subsume-id}
A causal effect \(\Pr_\x(\y)\) is identifiable (wrt \(\langle G, \V, \MM[\Pr(\V) > 0] \rangle\)) iff there is an AC for \(\Pr_\x(\y)\) on which the c-component method finds an invariant-cut.\footnote{We assume all variables are binary when constructing the AC.}
\end{theorem}

\section{Exploiting Constraints In Causal-Effect Identifiability: Examples}
\label{sec:examples}
We next use examples to demonstrate how our AC-based method can be applied to improve causal-effect identifiability under different types of constraints, some of which have been treated using methods dedicated to such constraints. In particular, we consider three types of constraints: (1) context-specific independencies; (2) functional dependencies; and (3) fully-known \(\Pr(\V),\) and show that we can exploit these constraints to simplify the constructed ACs in order to facilitate the discovery of invariant-cuts.

All the ACs considered in this section are compiled using the variable elimination method discussed in~\citep{Chavira.Darwiche.Ijcai.2007}, with constrained variable elimination orders: first eliminate all hidden variables \(\U,\) then observed variables \(\V\) in a bottom-up order. These ACs may be further simplified using the simplification rules in Appendix~\ref{app:simplify-AC}.

\subsection{Context-Specific Independence}
\label{sec:csi-constraint}

\begin{figure}[tb]
\centering
\includegraphics[width=\linewidth]{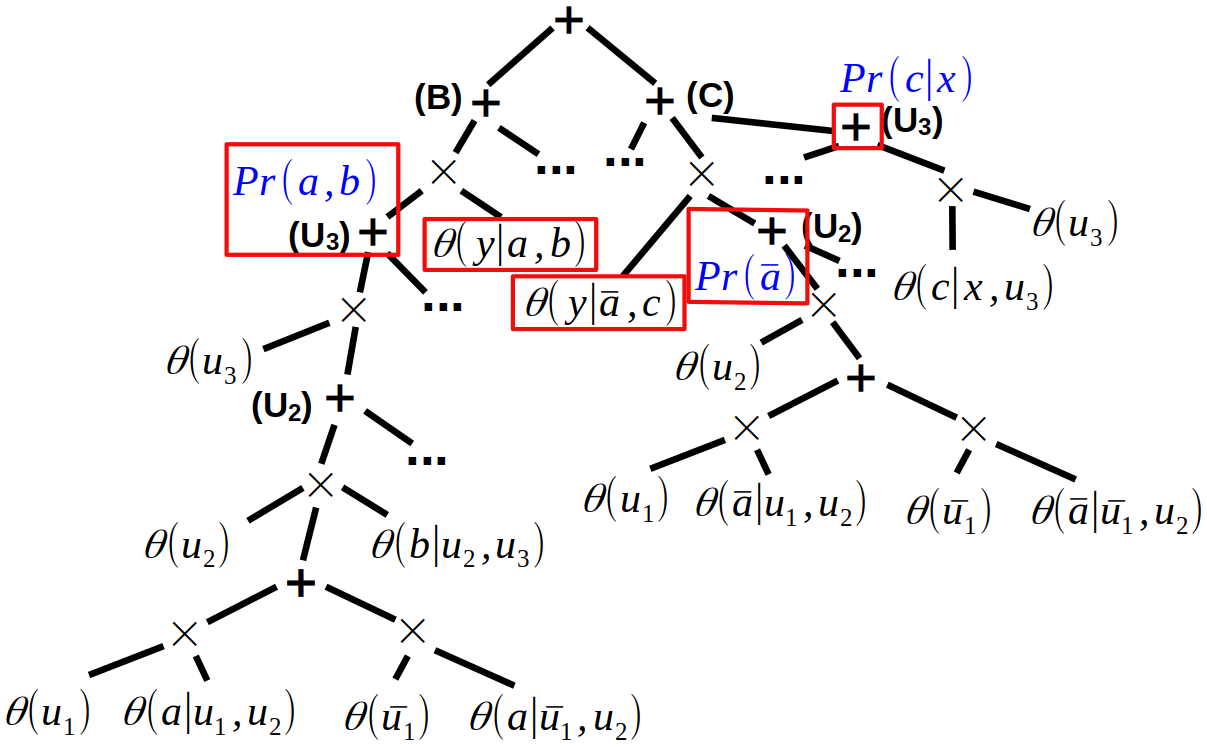}
\caption{AC for Figure~\ref{sfig:weak-weak1} with projections derived with the c-component method; projections are marked near the nodes in blue; the invariant-cut is marked with red boxes.}
\label{fig:id-circuit-int2}
\end{figure}

We first illustrate that the AC-based method can improve identifiability given knowledge of context-specific independence (CSIs)~\citep{uai/BoutilierFGK96, nips/TikkaHK19, pgm/ShenCD20}. CSI has the form (\(Y \indep X | \p\)) where \(X, \P\) are the parents of \(Y.\) That is, \(Y\) is independent of \(X\) when the values of other parents \(\P\) are set to \(\p.\) For simplicity, we write (\(Y \indep X | \P\)) if CSI holds under all instantiations of \(\P.\)

CSI constraints can be incorporated into the construction of ACs as follows. For each CSI relation (\(Y \indep X | \p\)), we merge all the parameters \(\theta_{y | x, \p}\) (which are leaves) in the AC into a single node \(\theta_{y | \p}.\) This is valid since \(\theta_{y | x, \p} = \Pr(y | \p)\) for all \(x\) and \(y\)~\citep{kr/Darwiche02}. We can then apply  simplification rules, such as those provided in Appendix~\ref{app:simplify-AC}, to facilitate the discovery of invariant-cuts.

Consider again the causal graph in Figure~\ref{sfig:weak-weak1} where the causal effect \(\Pr_x(y)\) is not identifiable under the positivity constraint \(\Pr(\V) > 0.\) Assume now the following CSI constraints \(\CC_{csi}\): \((Y \indep C | a, B)\) and \((Y \indep B | \bar{a}, C).\) We can construct the simplified AC in Figure~\ref{fig:id-circuit-int2} using the procedure above. We can then find an invariant-cut for the AC using the c-component method; see Appendix~\ref{app:derive-id-circuit-int2} for the derivation. Hence, we conclude that the causal effect is identifiable wrt \(\langle G, \V, \MM[\CC_{csi}, \Pr(\V) > 0] \rangle.\) We show another example adapted from~\citep{nips/TikkaHK19} in Appendix~\ref{app:csi-ex} to further illustrate the advantages brought by CSI constraints to identification, where our AC-based method yields the same identifying formula as the one found by the method in~\citep{nips/TikkaHK19}.

\subsection{Functional Dependencies}
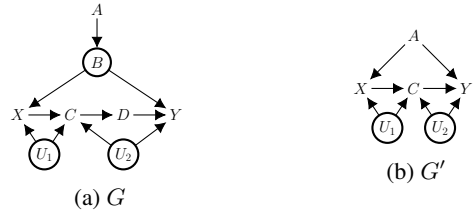
\begin{figure}[tb]
\centering
\begin{subfigure}[r]{0.49\linewidth}
\centering
\begin{tikzpicture}[->=stealth,auto,scale=\dagsize,transform shape]
\node[font=\huge] (A) at (1,0) {$A$};
\node[state,font=\huge] (B) at (1,-2) {$B$};
\node[font=\huge] (X) at (-2,-4) {$X$};
\node[font=\huge] (C) at (0,-4) {$C$};
\node[font=\huge] (D) at (2,-4) {$D$};
\node[font=\huge] (Y) at (4,-4) {$Y$};
\node[state,font=\huge] (U1) at (-1,-5.5) {$U_1$};
\node[state,font=\huge] (U2) at (2,-5.5) {$U_2$};

\path (A) edge (B);
\path (B) edge (X);
\path (B) edge (Y);
\path (X) edge (C);
\path (C) edge (D);
\path (D) edge (Y);
\path (U1) edge (X);
\path (U1) edge (C);
\path (U2) edge (C);
\path (U2) edge (Y);
\end{tikzpicture}
\caption{$G$}
\label{sfig:func-dep1}
\end{subfigure}
\begin{subfigure}[r]{0.49\linewidth}
\centering
\begin{tikzpicture}[->=stealth,auto,scale=\dagsize,transform shape]
\node[font=\huge] (A) at (0,-2) {$A$};
\node[font=\huge] (X) at (-2,-4) {$X$};
\node[font=\huge] (C) at (0,-4) {$C$};
\node[font=\huge] (Y) at (2,-4) {$Y$};
\node[state,font=\huge] (U1) at (-1,-5.5) {$U_1$};
\node[state,font=\huge] (U2) at (1,-5.5) {$U_2$};

\path (A) edge (X);
\path (A) edge (Y);
\path (X) edge (C);
\path (C) edge (Y);
\path (U1) edge (X);
\path (U1) edge (C);
\path (U2) edge (C);
\path (U2) edge (Y);
\end{tikzpicture}
\caption{$G'$}
\label{sfig:func-dep2}
\end{subfigure}
\caption{\(A,C,D,X,Y\) are observed; \(B,D\) are functional.}
\label{fig:func-dep}
\end{figure}

In a causal graph, a variable \(W\) with parents $\P$ is said to be \emph{functional,} 
or exhibit a \emph{functional dependency,} if \(\theta_{w | \p} \in \{0,1\}\) for all instantiations \(w, \p\) --- we assume that we do not know the specific values for these \(\theta_{w | \p}.\)
Functional dependencies have been recently exploited to improve the identifiability of causal effects~\citep{neurips/ChenDarwiche24} and the computation of causal queries~\citep{DarwicheECAI20b,uai/ChenDarwiche22,ijcai/HanCD23}. 

Let \(\W\) be a subset of variables that exhibit functional dependencies. \citep{neurips/ChenDarwiche24} showed that an unidentifiable causal effect may become identifiable given the functional variables \(\W,\) without the need of knowing the specific functions (\(\theta_{w | \p}\)'s) for \(\W.\) In addition, the work introduced a reduction-based method for testing identifiability under functional dependencies by eliminating (removing) a subset of functional variables \(\Q \subseteq \W\) from the causal graph.\footnote{The result holds if every \(Q \in \Q\) satisfies the following condition: \(Q \notin \X \cup \Y\) and every path from some hidden, non-functional variable to \(Q\) contains an observed variable (other than \(Q\)). See~\citep{neurips/ChenDarwiche24} for details about the elimination operation and its required positivity constraints.} That is, a causal effect \(\Pr_\x(\y)\) is identifiable wrt \(\langle G, \V, \MM \rangle\) iff it is identifiable wrt \(\langle G', \V', \MM' \rangle ,\) where \(\V' = \V \setminus \Q,\) \(G'\) is the result of eliminating \(\Q\) from \(G,\) and \(\MM'\) is obtained from \(\MM\) by changing functional variables \(\W\) to \((\W \setminus \Q).\) Hence, we can use the elimination of functional variables as a \emph{preprocessing mechanism} to reduce the identifiability problem with functional dependencies into classical identifiability (without functional dependencies).

When the states of variables are known, the if-direction of the result still holds. If we can show that \(\Pr_\x(\y)\) is identifiable wrt \(\langle G', \V', \MM' \rangle\) using the AC-based method, the same causal effect is guaranteed to be identifiable wrt the original \(\langle G, \V, \MM \rangle\) with functional dependencies. We next illustrate this with an example. Consider the causal graph \(G\) in Figure~\ref{sfig:func-dep1} with observed variables \(\V = \{A, C, D, X ,Y\}\) and the causal effect \(\Pr_x(y)\) wrt \(\langle G, \V, \MM[\Pr(A, C, X, Y) > 0, (C \indep U_1 | x, U_2), \W=\{B, D\}]\rangle.\) We first apply the preprocessing mechanism to eliminate the functional variables, which yields the causal graph \(G'\) in Figure~\ref{sfig:func-dep2}. Since there is no functional variables in \(G'\), we can now apply the AC-based method and show that \(\Pr_x(y)\) is identifiable wrt \(\langle G', \V'=\{A,C,X,Y\}, \MM'[\Pr(A, C, X, Y) > 0, (C \indep U_1 | x, U_2)]\rangle\); see Appendix~\ref{app:derive-func} for a detailed derivation. This implies that the causal effect is identifiable wrt the original identifiability tuple (with functional dependencies). 
Moreover, this example shows how the proposed framework can treat a combination of constraint types, in the form of CSI and functional dependencies.

\subsection{Fully-Known \(\Pr(\V)\)}
When an observational distribution \(\Pr^{\star}(\V)\) is given in addition to the causal graph \(G,\) we have an identifiability problem (wrt \(\langle G, \V, \MM[\Pr(\V) = \Pr^{\star}(\V)]\rangle\)) as defined previously.
This is quite common in practice when data over observed variables \(\V\) allows us to estimate \(\Pr^{\star}(\V)\) accurately.

Knowing \(\Pr^{\star}(\V)\) leads to several advantages. First, the positivity constraint over \(\Pr^{\star}(\V)\) becomes immediate since all the zero entries are determined by \(\Pr^{\star}(\V).\) Second, all the (well-defined) parameters \(\theta_{v|\p}\) over observed variables \(V, \P\) can be fixed to constant values in the AC. This allows us to exploit more parameter-specific knowledge such as 0/1 parameters and equal parameters (subsumes CSI) as exhibited by \(\Pr^{\star}(\V).\) One key observation is that a causal effect \(\Pr_\x(\y)\) is identifiable under a known \(\Pr^{\star}(\V)\) iff \(\Pr_\x(\y)\) always evaluates to a \emph{fixed constant} under all models that induce \(\Pr^{\star}(\V)\). This result follows directly from Proposition~\ref{prop:did}. In fact, the constant can always be evaluated by computing \(\Pr_\x(\y)\) under any single model that induces \(\Pr^{\star}(\V).\)

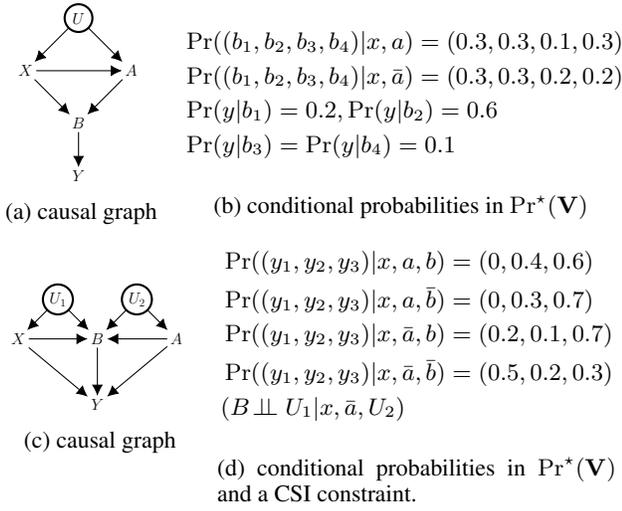
\begin{figure}[tb]
\centering
\begin{subfigure}[r]{0.3\linewidth}
\centering
\begin{tikzpicture}[->=stealth,auto,scale=\dagsize,transform shape]
\node[font=\huge] (X) at (0,0) {$X$};
\node[state,font=\huge] (U) at (2,2) {$U$};
\node[font=\huge] (A) at (4,0) {$A$};
\node[font=\huge] (B) at (2,-2) {$B$};
\node[font=\huge] (Y) at (2,-4) {$Y$};

\path (U) edge (X);
\path (U) edge (A);
\path (X) edge (A);
\path (X) edge (B);
\path (A) edge (B);
\path (B) edge (Y);
\end{tikzpicture}
\caption{causal graph}
\label{sfig:func-eg11}
\end{subfigure}
\begin{subfigure}[r]{0.69\linewidth}
\centering
\footnotesize
\begin{equation*}
\begin{split}
&\Pr((b_1,b_2,b_3,b_4) | x, a) = (0.3, 0.3, 0.1, 0.3)\\
&\Pr((b_1,b_2,b_3,b_4) | x, \bar{a}) = (0.3, 0.3, 0.2, 0.2)\\
&\Pr(y | b_1) = 0.2, \Pr(y | b_2) = 0.6\\
& \Pr(y | b_3) = \Pr(y | b_4) = 0.1\\
\end{split}
\end{equation*}
\caption{conditional probabilities in \(\Pr^\star(\V)\)}
\label{sfig:func-eg12}
\end{subfigure}
\begin{subfigure}[r]{0.35\linewidth}
\centering
\begin{tikzpicture}[->=stealth,auto,scale=\dagsize,transform shape]
\node[state,font=\huge] (U1) at (0,0) {$U_1$};
\node[font=\huge] (X) at (-1.5,-1.5) {$X$};
\node[font=\huge] (B) at (1.5,-1.5) {$B$};
\node[state,font=\huge] (U2) at (3,0) {$U_2$};
\node[font=\huge] (A) at (4.5,-1.5) {$A$};
\node[font=\huge] (Y) at (1.5,-4) {$Y$};

\path (U1) edge (X);
\path (U1) edge (B);
\path (U2) edge (B);
\path (U2) edge (A);
\path (X) edge (B);
\path (A) edge (B);
\path (X) edge (Y);
\path (B) edge (Y);
\path (A) edge (Y);
\end{tikzpicture}
\caption{causal graph}
\label{sfig:func-eg13}
\end{subfigure}
\begin{subfigure}[r]{0.63\linewidth}
\centering
\footnotesize
\begin{equation*}
\begin{split}
&\Pr((y_1,y_2,y_3) | x, a, b) = (0, 0.4, 0.6)\\
&\Pr((y_1,y_2,y_3) | x, a, \bar{b}) = (0, 0.3, 0.7)\\
&\Pr((y_1,y_2,y_3) | x, \bar{a}, b) = (0.2, 0.1, 0.7)\\
&\Pr((y_1,y_2,y_3) | x, \bar{a}, \bar{b}) = (0.5, 0.2, 0.3)\\
& (B \indep U_1|x,\bar{a},U_2)\\
\end{split}
\end{equation*}
\caption{conditional probabilities in \(\Pr^\star(\V)\) and a CSI constraint.}
\label{sfig:func-eg14}
\end{subfigure}
\caption{Examples for fully-known \(\Pr(\V).\)}
\label{fig:func-eg1}
\end{figure}

Consider the causal graph in Figure~\ref{sfig:func-eg11} where variable \(B\) has states \(b_1, b_2, b_3, b_4\) and all other variables are binary. Suppose \(\Pr^{\star}(\V) > 0\) satisfies the conditional probabilities in Figure~\ref{sfig:func-eg12}. We show that the causal effect \(\Pr_x(y)\) is identifiable under the fixed \(\Pr^{\star}(\V) > 0\) (this causal effect is not identifiable if $\Pr(\V)$ is not fixed to \(\Pr^{\star}(\V)\)).
Figure~\ref{fig:known-eg1} depicts the AC for the causal effect after plugging in the known constant values for parameters. If we further apply simplification rules on the AC, we obtain an AC which contains a single leaf with constant 0.28; see~Appendix~\ref{app:derive-func-eg12} for more details. That is, the causal effect \(\Pr_x(y) = 0.28\) for all models that induce \(\Pr^{\star}(\V)\) and is thus identifiable. 

The constraint of known \(\Pr(\V)\) can be combined with other types of constraints to further improve the identifiability of causal effects. We next show an example that assumes both known \(\Pr(\V)\) and context-specific independences. Consider the causal 
graph in Figure~\ref{sfig:func-eg13} where variable \(Y\) has three states \(y_1, y_2, y_3\) and all other variables are binary. Suppose the fixed \(\Pr^\star(\V)\) satisfies the conditional probabilities in Figure~\ref{sfig:func-eg14} and the CSI constraint (\(B \indep U_1|x,\bar{a},U_2\)). We next show that the causal effect \(\Pr_x(y_1)\) is identifiable under these constraints. Note that strict positivity no longer holds in this example due to the zero entries in the CPT of \(Y.\) In fact, the existence of zero entries enables the identifiability since it allows us to prune nodes that evaluate to zero from the AC. Figure~\ref{sfig:known-eg12} depicts the AC for \(\Pr_x(y_1)\) after applying the simplification rules in Appendix~\ref{app:simplify-AC} along with the projections and invariant-cut. Similar to the examples in Figure~\ref{fig:func-dep} and Appendix~\ref{app:csi-ex}, we need to modify the c-component method slightly to find the projections; see Appendix~\ref{app:derive-known-pr} for details. The identifying formula in this case is
\(\Pr_x(y_1) = 0.2\Pr(\bar{a},b|x) + 0.5\Pr(\bar{a},\bar{b}|x)\) which can be evaluated using the known \(\Pr^\star(\V).\)

\begin{figure}[tb]
\centering
\begin{subfigure}[r]{0.49\linewidth}
\centering
\includegraphics[width=\linewidth]{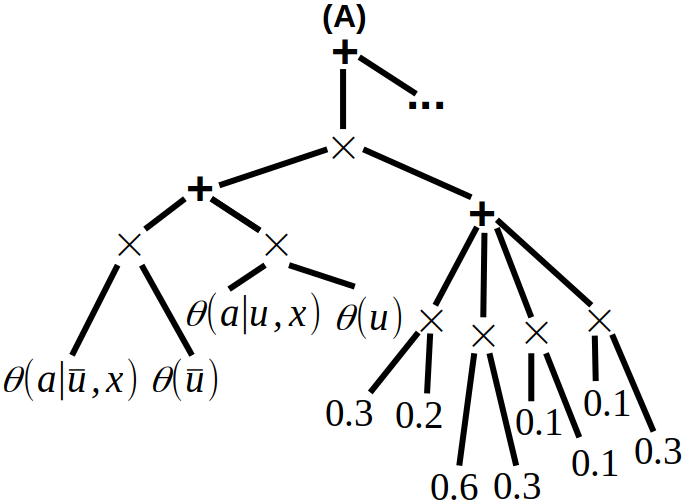}
\caption{}
\label{sfig:known-eg11}
\end{subfigure}
\begin{subfigure}[r]{0.49\linewidth}
\centering
\includegraphics[width=\linewidth]{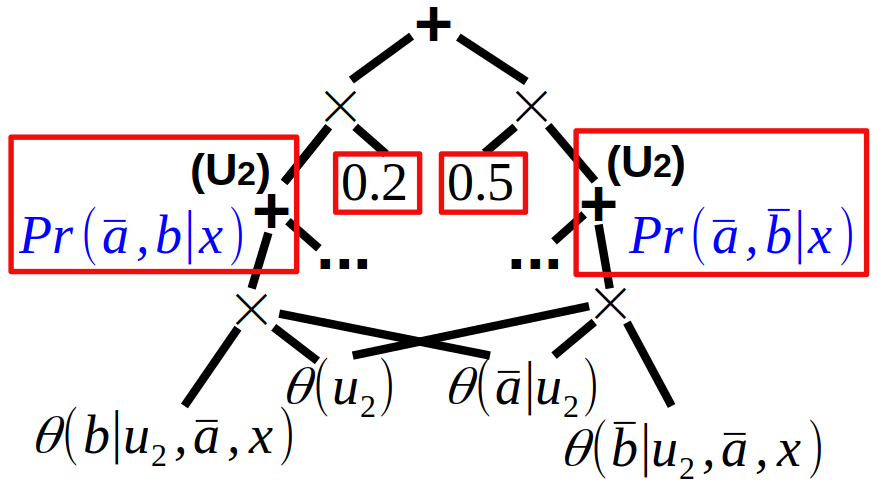}
\caption{}
\label{sfig:known-eg12}
\end{subfigure}
\caption{ACs constructed for Figure~\ref{sfig:func-eg11} and Figure~\ref{sfig:func-eg13} under known \(\Pr^{\star}(\V).\)}
\label{fig:known-eg1}
\end{figure}

\section{Conclusion}
\label{sec:conclusion}
We formalized the notion of constrained identifiability which targets causal-effect identifiability in the presence of various types of constraints. We also proposed an approach for testing constrained identifiability based on constructing an Arithmetic Circuit (AC) for computing the causal effect using the causal graph and available constraints, and then finding invariant-cuts. We showed that this method is at least as complete as classical methods such as the ID algorithm and the do-calculus (which handle strict positivity constraints). We further demonstrated with examples how this approach can be used for testing causal-effect identifiability under various types of constraints, including context-specific independencies, functional dependencies, and fully-known observational distributions,
therefore improving the completeness of classical methods when going beyond strict positivity constraints.

\bibliography{aaai25}

\newpage
\appendix
\section{Simplifying Arithmetic Circuits}
\label{app:simplify-AC}
The constructed AC can be simplified (reduced) by exploiting parameter-specific information such as 0/1 entries and equal parameters. While most of the simplifications can be handled by the \textsc{ace} package,\footnote{http://reasoning.cs.ucla.edu/ace.} we introduce three more rules that may lead to additional simplifications:

\begin{enumerate}
\item (sum-out) For a \(+\)-node in the AC, if it computes \(\sum_{z} \Pr(z | \w)\), replace the \(+\)-node with constant \(1.\)
\item (merge) If two nodes in the AC are same and have a same set of children, merge them into a single node.
\item (distributivity) For a \(+\)-node in the AC, if it computes \(\sum_i (c \cdot s_i)\), simplify it to \(c \cdot \sum_i s_i.\)
\end{enumerate}

These rules can be applied recursively to simplify the AC to the maximum extent. Note that we can always omit AC nodes that evaluate to zero and nodes that evaluate to one if their parent is a \(\times\)-node. Please check out examples and their derivations in Appendix~\ref{app:examples-derivations} for more details.

\section{C-component Method For Finding Projections}
\label{app:c-component}
The c-component method finds projections by decomposing a causal graph into (maximal) c-components~\citep{uai/TianP02b,TianPearl03,aaai/ShpitserP06}.\footnote{Any causal graph can be decomposed into c-components as shown in~\citep{uai/TianP02b,TianPearl03}, though it is common to consider semi-Markovian models in which all hidden variables are roots and have two children.} 
As shown in the main paper, a causal graph induces a set of \(\V\)-computable subgraphs.
Moreover, each \(\V\)-computable subgraph \(S\) induces a c-factor~\citep{uai/TianP02b,TianPearl03,aaai/HuangV06} which we denote as \(Q[S].\) For example, the \(\V\)-computable subgraphs in Figure~\ref{sfig:c-components2} correspond to the following formulas:

\begin{itemize}
\item \(Q[S_1] = \sum_{U_1,U_3} \theta_{U_1}\theta_{U_3} \theta_{V_1 | U_1, U_3} \theta_{V_3 | U_1, V_2} \theta_{V_5 | U_3, V_4}\)
\item \(Q[S_2] = \sum_{U_1,U_3} \theta_{U_1}\theta_{U_3} \theta_{V_1 | U_1, U_3} \theta_{V_3 | U_1, V_2}\)
\item \(Q[S_3] = \sum_{U_1} \theta_{U_1}\theta_{V_3 | U_1, V_2}\)
\item \(Q[S_4] = \sum_{U_2} \theta_{U_2}\theta_{V_2 | U_2, V_1}\theta_{V_4 | U_2, V_3}\)
\item \(Q[S_5] = \sum_{U_2} \theta_{U_2}\theta_{V_4 | U_2, V_3}\)
\end{itemize}

To see why the c-factors for \(\V\)-computable subgraphs can be computed from \(\Pr(\V),\)
consider the following inductive argument. For the original causal graph \(G,\) \(Q[G] = \Pr(\V)\) which is a projection. Whenever we prune a leaf node \(W\) from a \(\V\)-computable graph \(G\) to obtain a new graph \(G'\) (case~(3) in Definition~\ref{def:id-ccomp}), the c-factor for \(G'\) is equal to \(Q[G'] = \sum_W Q[G]\), which has a projection since \(Q[G]\) has a projection. Whenever we decompose a \(\V\)-computable subgraph \(G\) into c-components \(\SS = \{S_1, \dots, S_k\}\) (case~(2) in Definition~\ref{def:id-ccomp}), we can compute each \(Q[S]\) for each \(S \in \SS\) from \(Q[G]\) using the ``Generalized Q-decomposition'' in~\citep{TianPearl03}. Hence, each \(Q[S]\) has a projection since \(Q[G]\) has a projection.

The projections for the c-factors of \(\V\)-computable subgraphs in Figure~\ref{sfig:c-components2} are shown below, where \(\equiv\) denotes the equivalence between the c-factor and projection.
\begin{itemize}
\item \(Q[S_1] \equiv \Pr(V_1) \Pr(V_3 | V_1, V_2) \Pr(V_5 | V_1, V_2, V_3, V_4)\)
\item \(Q[S_2] \equiv \Pr(V_1) \Pr(V_3 | V_1, V_2)\)
\item \(Q[S_3] \equiv \sum_{V_1} \Pr(V_1) \Pr(V_3 | V_1, V_2)\)
\item \(Q[S_4] \equiv \Pr(V_2 | V_1) \Pr(V_4 | V_1, V_2, V_3)\)
\item \(Q[S_5] \equiv \sum_{V_2} \Pr(V_2 | V_1) \Pr(V_4 | V_1, V_2, V_3)\)
\end{itemize}

\section{Examples and Derivations}
\label{app:examples-derivations}
\subsection{Derivation For Figure~\ref{fig:id-circuit-int2}}
\label{app:derive-id-circuit-int2}
Figure~\ref{fig:ac-construct} depicts the original AC for \(\Pr_x(y)\) under the CSI constraints after merging \(\theta_{y | a,b,c}\) and \(\theta_{y | a, b, \bar{c}}\) into \(\theta_{y | a, b}\), and merging \(\theta_{y | \bar{a}, b, c}\) and \(\theta_{y | \bar{a}, \bar{b}, c}\) into \(\theta_{y | \bar{a}, c}\) (merged nodes are color-coded in blue). We can further simplify the AC by pushing the \(+\)-nodes downwards (distributivity rule) and summing-out the leaf nodes (sum-out rule). We color-coded the removed \(+\)-nodes and leaf nodes in red in Figure~\ref{fig:ac-construct}. Finally, we apply the distributivity rule on \(+ (U_3)\)-node on the right branch, which yields the simplified AC in Figure~\ref{fig:id-circuit-int2}.

We can now apply the c-component method to derive projections for the AC nodes. In particular, The causal graph in Figure~\ref{sfig:weak-weak1} induces the following c-factors:
\begin{equation*}
\begin{split}
&Q[S_1] = \sum_{U_1,U_2,U_3} \theta_{U_1}\theta_{U_2}\theta_{U_3}\theta_{A | U_1,U_2}\theta_{B | U_2,U_3} \equiv \Pr(A,B)\\
&Q[S_2] = \sum_{U_1,U_2} \theta_{U_1}\theta_{U_2}\theta_{A | U_1,U_2} \equiv \Pr(A)\\
&Q[S_3] = \sum_{U_3} \theta_{U_3}\theta_{C | U_3,X} \equiv \Pr(C|X)\\
\end{split}
\end{equation*}
Each of the above c-factors matches some node in the AC. Moreover, after plugging in the projections for the AC nodes, we find an invariant-cut as shown in Figure~\ref{fig:id-circuit-int2}. We thus conclude that the causal effect is identifiable.

\begin{figure}[tb]
\centering
\includegraphics[width=\linewidth]{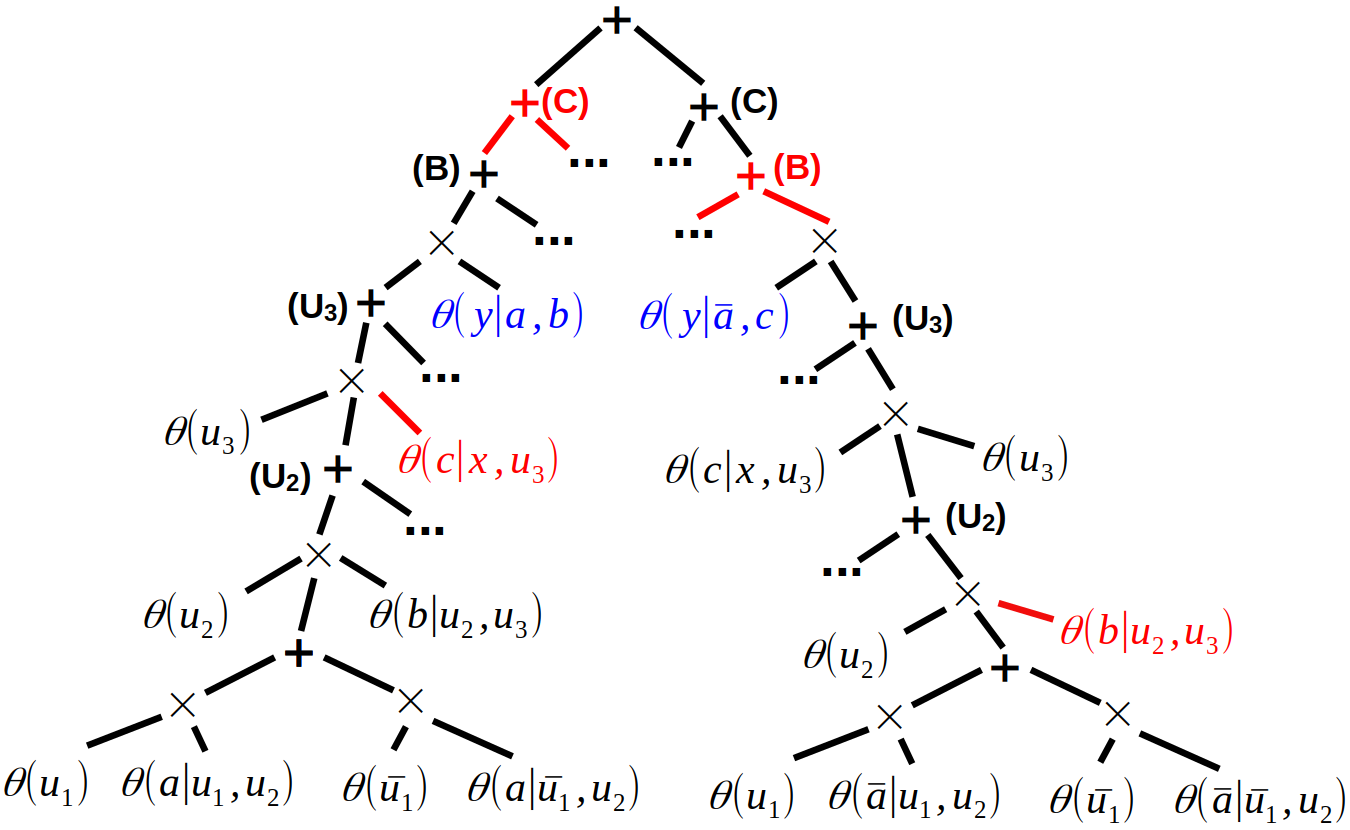}
\caption{Original AC compiled from Figure~\ref{sfig:weak-weak1} under CSI constraints \((Y\indep C|A=a, B)\) and \((Y \indep B|A=\bar{a}, C).\)}
\label{fig:ac-construct}
\end{figure}

\begin{figure}[tb]
\centering
\begin{subfigure}[r]{0.3\linewidth}
\centering
\begin{tikzpicture}[->=stealth,auto,scale=\dagsize,transform shape]
\node[state,font=\huge] (U1) at (0.5,0) {$U_1$};
\node[font=\huge] (A) at (1.5,-1.5) {$A$};
\node[state,font=\huge] (U2) at (3,0) {$U_2$};
\node[font=\huge] (B) at (4.5,-1.5) {$B$};
\node[state,font=\huge] (U3) at (5.5,0) {$U_3$};

\path (U1) edge (A);
\path (U2) edge (A);
\path (U2) edge (B);
\path (U3) edge (B);
\end{tikzpicture}
\caption{$S_1$}
\label{sfig:csi1-comp-1}
\end{subfigure}
\begin{subfigure}[r]{0.3\linewidth}
\centering
\begin{tikzpicture}[->=stealth,auto,scale=\dagsize,transform shape]
\node[state,font=\huge] (U1) at (0,0) {$U_1$};
\node[font=\huge] (A) at (1.5,-1.5) {$A$};
\node[state,font=\huge] (U2) at (3,0) {$U_2$};

\path (U1) edge (A);
\path (U2) edge (A);
\end{tikzpicture}
\caption{$S_2$}
\label{sfig:csi1-comp-2}
\end{subfigure}
\begin{subfigure}[r]{0.3\linewidth}
\centering
\begin{tikzpicture}[->=stealth,auto,scale=\dagsize,transform shape]
\node[state,font=\huge] (U3) at (0,0) {$U_3$};
\node[font=\huge] (C) at (0,-2) {$C$};

\path (U3) edge (C);
\end{tikzpicture}
\caption{$S_3$}
\label{sfig:csi1-comp-3}
\end{subfigure}
\caption{\(\V\)-computable subgraphs for Figure~\ref{sfig:weak-weak1}.}
\label{fig:csi-compt}
\end{figure}
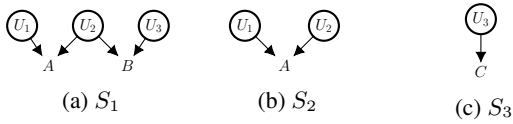

\subsection{Additional Example on CSI}
\label{app:csi-ex}
The following example is adapted from~\citep{nips/TikkaHK19}. Consider the causal graph \(G\) in Figure~\ref{sfig:csi-ex1} with \(\V = \{A, X, Y\}.\) The causal effect \(\Pr_x(y)\) is not identifiable under positivity constraint \(\Pr(\V) > 0\). However, the causal effect becomes identifiable under the CSI constraints (\(X \indep U | \bar{a}\)) and (\(Y \indep U | a, X\)). Figure~\ref{sfig:csi-ex2} depicts an AC with an invariant-cut for the causal effect.

We can find the projections in Figure~\ref{sfig:csi-ex2} using the c-component method \emph{but with a slight modification to leverage the CSI constraints.}
In particular, consider the \(+\)-node on the right branch, which has the following expression:
\[\theta_{u}\theta_{y | u, x, \bar{a}}  + \theta_{\bar{u}}\theta_{y | \bar{u}, x, \bar{a}}\]
This expression does not directly match any c-factors induced by \(\V\)-computable subgraphs. However, the c-factor for the \(\V\)-computable subgraph \(S\) in Figure~\ref{sfig:csi-ex3}
\begin{equation*}
\begin{split}
&\sum_U \theta_{U} \theta_{Y | U, X, A}\theta_{X | U, A} \equiv  \Pr(X, Y | A)
\end{split}
\end{equation*}
can be further simplified when we fix \(A = \bar{a}.\) In particular, the CSI constraint (\(X \indep U | \bar{a}\)) allows us to replace \(\theta_{X | \bar{a},U}\) with \(\theta_{X | \bar{a}}\) in the c-factor, which yields
\[\sum_u \theta_{u} \theta_{Y | u, X, \bar{a}} \theta_{X | \bar{a}} \equiv  \Pr(X,Y | \bar{a}) \]
That is, the new graph \(S'\) in Figure~\ref{sfig:csi-ex4} is also \(\V\)-computable when we fix \(A = \bar{a}.\)
This allows us to further decompose \(S'\) into c-components by Definition~\ref{def:id-ccomp} and attain the following c-factor and corresponding projection:
\begin{equation*}
\begin{split}
&\sum_u \theta_{u}\theta_{Y | u, X, \bar{a}} \equiv \Pr(Y | X, \bar{a})\\
\end{split}
\end{equation*}
This example shows that CSI constraints may empower additional c-component decompositions and hence discover a larger set of \(\V\)-computable subgraphs.

\begin{figure}[tb]
\centering
\begin{subfigure}[r]{0.3\linewidth}
\centering
\begin{tikzpicture}[->=stealth,auto,scale=\dagsize,transform shape]
\node[state,font=\huge] (U) at (0,0) {$U$};
\node[font=\huge] (X) at (-2,-2) {$X$};
\node[font=\huge] (Y) at (2,-2) {$Y$};
\node[font=\huge] (A) at (0,-4) {$A$};

\path (U) edge (X);
\path (U) edge (Y);
\path (X) edge (Y);
\path (A) edge (X);
\path (A) edge (Y);
\end{tikzpicture}
\caption{causal graph}
\label{sfig:csi-ex1}
\end{subfigure}
\begin{subfigure}[r]{0.69\linewidth}
\centering
\includegraphics[width=\linewidth]{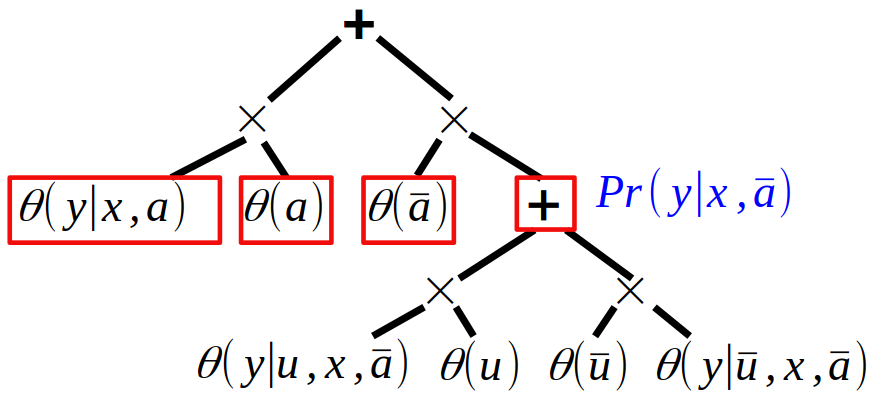}
\caption{ID-AC}
\label{sfig:csi-ex2}
\end{subfigure}
\begin{subfigure}[r]{0.49\linewidth}
\centering
\begin{tikzpicture}[->=stealth,auto,scale=\dagsize,transform shape]
\node[state,font=\huge] (U) at (0,0) {$U$};
\node[font=\huge] (X) at (-2,-2) {$X$};
\node[font=\huge] (Y) at (2,-2) {$Y$};

\path (U) edge (X);
\path (U) edge (Y);
\path (X) edge (Y);
\end{tikzpicture}
\caption{$S$}
\label{sfig:csi-ex3}
\end{subfigure}
\begin{subfigure}[r]{0.49\linewidth}
\centering
\begin{tikzpicture}[->=stealth,auto,scale=\dagsize,transform shape]
\node[state,font=\huge] (U) at (0,0) {$U$};
\node[font=\huge] (X) at (-2,-2) {$X$};
\node[font=\huge] (Y) at (2,-2) {$Y$};

\path (U) edge (Y);
\path (X) edge (Y);
\end{tikzpicture}
\caption{$S'$}
\label{sfig:csi-ex4}
\end{subfigure}
\caption{CSI example adapted from~\citep{nips/TikkaHK19} and an AC with an invariant-cut for \(\Pr_x(y).\)}
\label{fig:c-components-set1}
\end{figure}

\subsection{Derivation For Figure~\ref{fig:func-dep}}
\label{app:derive-func}
We show that the causal effect \(\Pr_x(y)\) is identifiable wrt the causal graph \(G\) in Figure~\ref{sfig:func-dep2} under the CSI constraint \((C \indep U_1 | x, U_2).\)
We first construct the AC for \(\Pr_x(y)\) in Figure~\ref{sfig:func-ex1}, which incorporates the CSI constraint with the method in Section~\ref{sec:csi-constraint}. We then simplify the AC using the simplification rules in Appendix~\ref{app:simplify-AC}, which yields Figure~\ref{sfig:func-ex2}. 

To find an invariant-cut, we need to find a projection for the following expression:
\begin{equation*}
\begin{split}
\sum_{u_2} \theta_{u_2}\theta_{c|u_2,x}\theta_{y | u_2,A,C}
\end{split}
\end{equation*}
which does not correspond to any c-factors directly. We now consider the c-factor of the \(\V\)-computable subgraph $S$ in Figure~\ref{sfig:func-ex3}
\begin{equation*}
\begin{split}
Q[S] &= \sum_{U_1,U_2} \theta_{U_1}\theta_{U_2}\theta_{X|U_1,A}\theta_{C|U_1,U_2,X}\theta_{Y | U_2,A,C} \\
& \equiv \Pr(X,C,Y | A)\\
\end{split}
\end{equation*}
If we fix \(X=x,\) we can leverage the CSI constraint and replace \(\theta_{C|U_1,U_2,x}\) with \(\theta_{C | U_2, x},\) which yields
\begin{equation*}
\begin{split}
&\sum_{u_1,u_2} \theta_{u_1}\theta_{u_2}\theta_{x|u_1,A}\theta_{C|u_1,u_2,x}\theta_{Y | u_2,A,C} \\
&= \sum_{u_1,u_2} \theta_{u_1}\theta_{u_2}\theta_{x|u_1,A}\theta_{C|u_2,x}\theta_{Y | u_2,A,C}\\
& \equiv \Pr(x,C,Y | A)\\
\end{split}
\end{equation*}
Hence, the subgraph \(S'\) in Figure~\ref{sfig:csi-ex4} is also \(\V\)-computable under \(X=x.\) We can further decompose \(S'\) into c-components and obtain the following c-factor:
\begin{equation*}
\begin{split}
\sum_{u_2} \theta_{u_2}\theta_{C|u_2,x}\theta_{Y | u_2,A,C} \equiv \Pr(C,Y | A,x)
\end{split}
\end{equation*}
This c-factor can now be used as a projection for the \(+ (U_2)\)-node in the AC in Figure~\ref{sfig:func-dep2}. We can now find an invariant-cut (marked in red) and get an identifying formula \(\Pr_x(y) = \sum_a \Pr(a) \sum_c \Pr(c,y | a,x).\)

\begin{figure}[tb]
\centering
\begin{subfigure}[r]{0.30\linewidth}
\centering
\includegraphics[width=\linewidth]{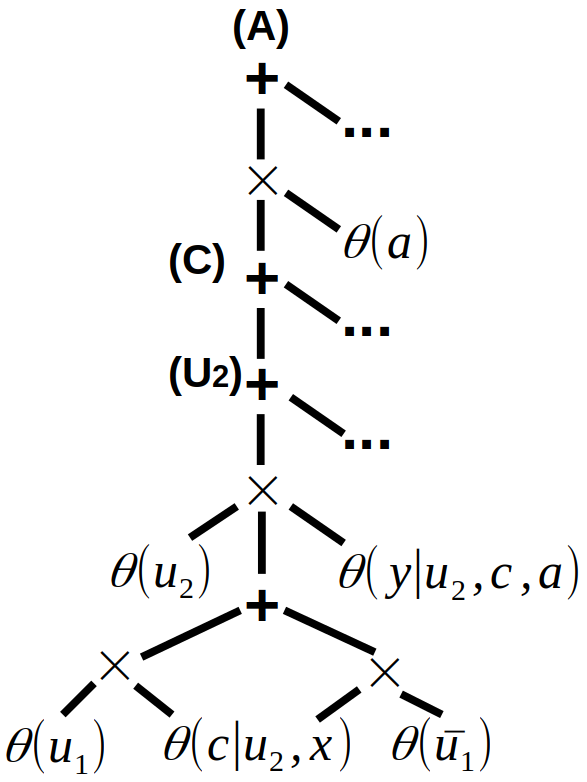}
\caption{AC}
\label{sfig:func-ex1}
\end{subfigure}
\begin{subfigure}[r]{0.30\linewidth}
\centering
\includegraphics[width=0.9\linewidth]{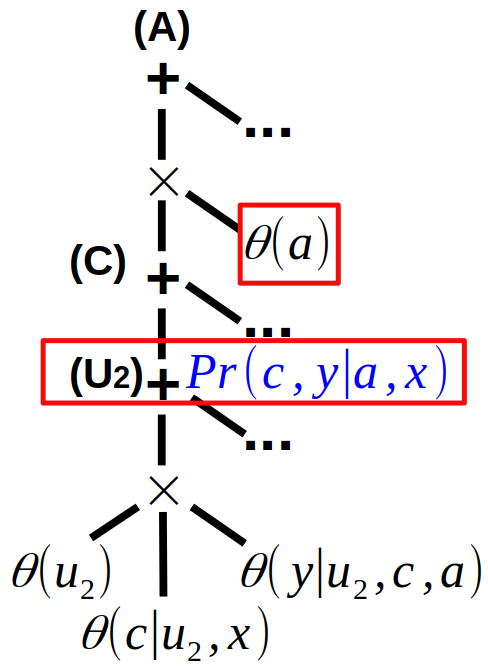}
\caption{Simplified AC}
\label{sfig:func-ex2}
\end{subfigure}
\begin{minipage}[s][2.5cm]{.38\linewidth}
\begin{subfigure}[r]{\linewidth}
\centering
\begin{tikzpicture}[->=stealth,auto,scale=\dagsize,transform shape]
\node[font=\huge] (X) at (-2,-4) {$X$};
\node[font=\huge] (C) at (0,-4) {$C$};
\node[font=\huge] (Y) at (2,-4) {$Y$};
\node[state,font=\huge] (U1) at (-1,-5.5) {$U_1$};
\node[state,font=\huge] (U2) at (1,-5.5) {$U_2$};

\path (X) edge (C);
\path (C) edge (Y);
\path (U1) edge (X);
\path (U1) edge (C);
\path (U2) edge (C);
\path (U2) edge (Y);
\end{tikzpicture}
\caption{$S$}
\label{sfig:func-ex3}
\end{subfigure}
\vfil
\begin{subfigure}[r]{\linewidth}
\centering
\begin{tikzpicture}[->=stealth,auto,scale=\dagsize,transform shape]
\node[font=\huge] (X) at (-2,-4) {$X$};
\node[font=\huge] (C) at (0,-4) {$C$};
\node[font=\huge] (Y) at (2,-4) {$Y$};
\node[state,font=\huge] (U1) at (-1,-5.5) {$U_1$};
\node[state,font=\huge] (U2) at (1,-5.5) {$U_2$};

\path (X) edge (C);
\path (C) edge (Y);
\path (U1) edge (X);
\path (U2) edge (C);
\path (U2) edge (Y);
\end{tikzpicture}
\caption{$S'$}
\label{sfig:func-ex4}
\end{subfigure}
\end{minipage}

\caption{(Simplified) AC constructed for \(\Pr_x(y)\) wrt the causal graph in Figure~\ref{sfig:func-dep2}.}
\label{fig:c-components-set2}
\end{figure}
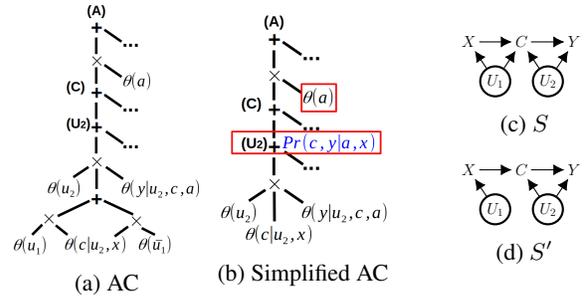

\subsection{Derivation For Figure~\ref{sfig:func-eg12}}
\label{app:derive-func-eg12}
We show that the AC in Figure~\ref{fig:known-eg1} constructed for \(\Pr_x(y)\) under \(\Pr^\star(\V)\) in Figure~\ref{sfig:func-eg12} can be simplified to an AC with a single constant. First, by evaluating the constants and merging the nodes, we obtain the AC in Figure~\ref{fig:known-eg1-simp}. If we further push down the \(+\)-node at the root using the distributivity rule and apply the sum-out rule, we obtain an AC with a single constant \(0.28.\)

\begin{figure}[tb]
\centering
\includegraphics[width=0.8\linewidth]{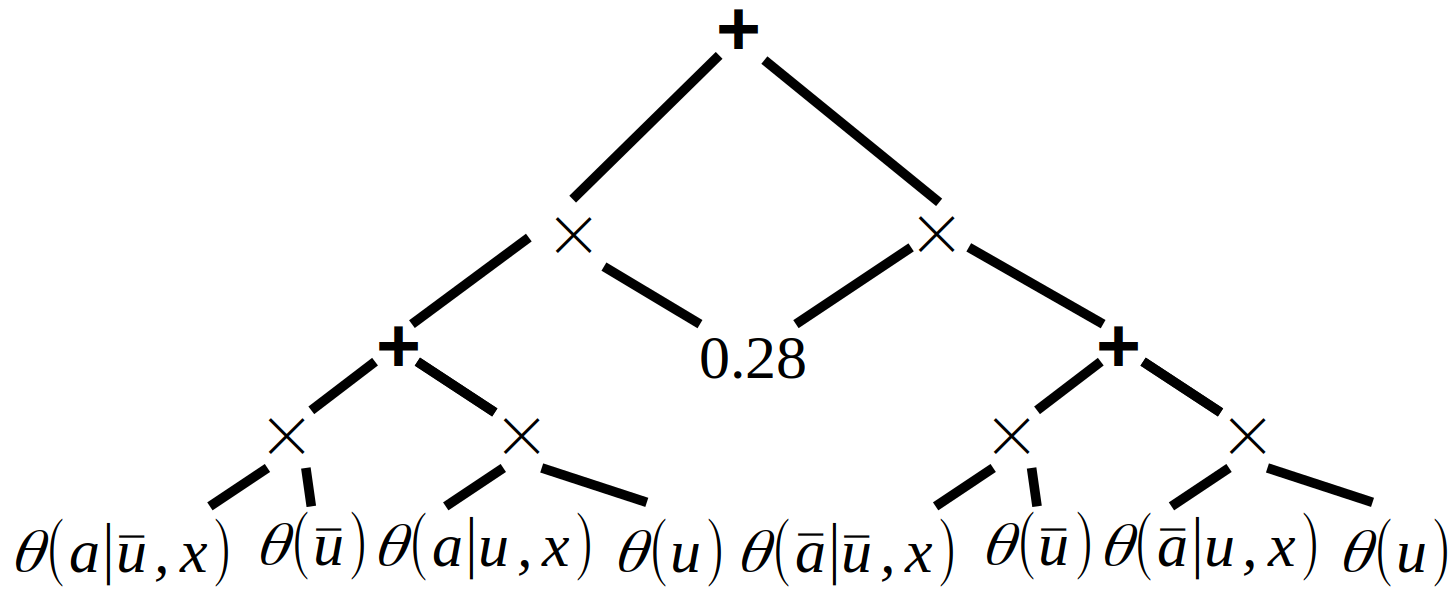}
\caption{Simplified AC from Figure~\ref{fig:known-eg1}.}
\label{fig:known-eg1-simp}
\end{figure}

\subsection{Derivation For Figure~\ref{sfig:func-eg13}}
\label{app:derive-known-pr}
Given the AC in Figure~\ref{sfig:known-eg12}, we still need to find the projection for the \(+\)-nodes with the following expression:
$\sum_{u_2} \theta_{u_2} \theta_{\bar{a} | u_2} \theta_{B | u_2, \bar{a}, x}$,
which does not match any c-factors for \(\V\)-computable subgraphs. However, by exploiting the CSI relation, we can set \(A=\bar{a}, X=x\) in the following c-factor 
$\sum_{U_1,U_2} \theta_{U_1}\theta_{U_2} \theta_{X | U_1}\theta_{A | U_2} \theta_{B | U_1, U_2, X, A}$ and get
\begin{equation*}
\begin{split}
&\sum_{u_1,u_2} \theta_{u_1}\theta_{u_2} \theta_{x | u_1}\theta_{\bar{a} | u_2} \theta_{B | u_1, u_2, x, \bar{a}}\\
& = \sum_{u_1,u_2} \theta_{u_1}\theta_{u_2} \theta_{x | u_1}\theta_{\bar{a} | u_2} \theta_{B | u_2, x, \bar{a}}  \equiv  \Pr(x, \bar{a}, B)\\
\end{split}
\end{equation*}
That is, once we fixed the values of \(X\) and \(A,\) the c-component \(S\) in Figure~\ref{fig:func-eg13-subgraph}. We can now further decompose \(S\) into two c-component, one of which is exactly the expression for the \(+\)-node in the AC:
\[\sum_{u_2} \theta_{u_2} \theta_{\bar{a} | u_2} \theta_{B | u_2, x, \bar{a}}  \equiv  \Pr(\bar{a}, B | x)\]
Hence, we conclude the projection fro the \(+\)-node is \(\Pr(\bar{a}, b | x).\)

\begin{figure}[tb]
\centering
\begin{tikzpicture}[->=stealth,auto,scale=\dagsize,transform shape]
\node[state,font=\huge] (U1) at (0,0) {$U_1$};
\node[font=\huge] (X) at (-1.5,-1.5) {$X$};
\node[font=\huge] (B) at (1.5,-1.5) {$B$};
\node[state,font=\huge] (U2) at (3,0) {$U_2$};
\node[font=\huge] (A) at (4.5,-1.5) {$A$};

\path (U1) edge (X);
\path (U2) edge (B);
\path (U2) edge (A);
\path (X) edge (B);
\path (A) edge (B);
\end{tikzpicture}
\caption{\(\V\)-computable subgraph for Figure~\ref{sfig:func-eg13} under \(X=x, A=a.\)}
\label{fig:func-eg13-subgraph}
\end{figure}
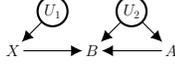

\section{Proofs} \label{app:proofs}
\subsection*{Proof of Proposition~\ref{prop:weak-weak-1}}
\begin{proof}
The causal effect is unidentifiable under some \(\Pr^\star(\V) > 0\) since the ID algorithm returns ``unidentifiable'' (i.e., the graph contains a \emph{hedge}\footnote{Hedge is a complete graphical structure that induces unidentifiability; see~\citep{aaai/ShpitserP06} for details.}). We next show that the causal effect \(\Pr_x(y)\) is identifiable whenever \(\Pr^\star(\V)\) satisfies the CSI constraints \((Y \indep C|a, B)\) and \((Y \indep B|\bar{a}, C).\) We show that the \(\Pr_x(y)\) can be identified as \[\Pr_x(y) = \Pr(a,y) + \sum_c \Pr(y|\bar{a},c)\Pr(c|x)\Pr(\bar{a})\]

We break \(\Pr_x(y)\) into two parts \(\Pr_x(a,y)\) and \(\Pr_x(\bar{a},y)\) by the law of total probability and identify each part separately. We first consider \(\Pr_x(a,y).\)
\begin{equation*}
\begin{split}
&\sum_{b}\sum_c \Pr(y | a,b,c) \sum_{u_1}\sum_{u_2}\Pr(u_1)\Pr(u_2)\\ 
&\Pr(a|u_1,u_2) \sum_{u_3} \Pr(u_3)\Pr(c|u_3,x)\Pr(b|u_2,u_3)\\
&=\sum_{b} \Pr(y | a,b) \sum_{u_1} \sum_{u_2}\Pr(u_1)\Pr(u_2)\Pr(a|u_1,u_2)\\ 
&\sum_{u_3} \Pr(u_3)\Pr(b|u_2,u_3)\;\;\;\;\;\; (Y \indep C | a, B)\\ 
&=\sum_{b} \Pr(y | a,b) \sum_{u_1} \sum_{u_2}\Pr(u_1)\Pr(u_2)\Pr(a|u_1,u_2)\Pr(b|u_2)\\
&\;\;\;\;\;\;\;\;\;\;\;\; (U_2 \indep U_3)\\ 
&= \sum_{b} \Pr(y | a,b) \sum_{u_2}\Pr(u_2)\Pr(a|u_2)\Pr(b|u_2)\\ 
&\;\;\;\;\;\;\;\;\;\;\;\; (U_1 \indep U_2)\\ 
&= \sum_{b} \Pr(y | a,b) \Pr(a,b) \;\;\;\;\;\;\; (A \indep B \ |\ U_2)\\ 
&= \Pr(a,y)\\ 
\end{split}
\end{equation*}
We next consider \(\Pr_x(\bar{a},y)\).
\begin{equation*}
\begin{split}
&\sum_{b}\sum_c \Pr(y | \bar{a},b,c) \sum_{u_1}\sum_{u_2}\Pr(u_1)\Pr(u_2)\Pr(\bar{a}|u_1,u_2)\\ 
&\ \ \ \sum_{u_3} \Pr(u_3)\Pr(c|u_3,x)\Pr(b|u_2,u_3)\\
&= \sum_{c} \Pr(y | \bar{a},c) \sum_{u_1} \sum_{u_2}\Pr(u_1)\Pr(u_2)\Pr(\bar{a}|u_1,u_2)\\ 
&\ \ \ \sum_{u_3} \Pr(u_3)\Pr(c|u_3,x) \;\;\;\;\;\;\; (Y \indep B\ |\ A=\bar{a}, C)\\
&=\sum_{c} \Pr(y | \bar{a},c) \sum_{u_1} \sum_{u_2}\Pr(u_1)\Pr(u_2)\Pr(\bar{a}|u_1,u_2) \Pr(c|x)\\
&\;\;\;\;\;\;\;\;\;\;\;\; (U_3 \indep X) \\
&= \sum_{c}\Pr(y | \bar{a},c)\Pr(\bar{a})\Pr(c|x)\;\;\;\;\;\;\;\;\; (U_1 \indep U_2)\\
\end{split}
\end{equation*}
Hence, \(\Pr_x(y)\) is identifiable when the CSI constraints hold.
\end{proof}

\subsection*{Proof of Proposition~\ref{prop:strong-uid1}}
WLG, to show that the causal effect \(\Pr_x(y)\) is unidentifiable under any \(\Pr^\star(X,Y) > 0\), we construct two parameterizations \(\Theta^1\) and \(\Theta^2\) that agree on any \(\Pr^\star(X,Y)\) but disagrees on \(\Pr_{x_1}({y_1})\), where \(\{x_1, \dots, x_n\}\) are the states of \(X\) and \(\{y_1, \dots, y_m\}\) are the states of \(Y\). Moreover, we assume that the hidden variables \(U\) is binary and has two states \(u_1, u_2\).
We start by constructing a parameterization for \(G\) that satisfy certain conditions. 
\begin{lemma}
\label{lem:string-uid1-1}
There exists a parameterization \(\Theta\) for \(G\) that induces \(\Pr^\star(X,Y)\) and in which (i) all CPTs are strictly positive; and (ii) \(\theta_{x | u_1} \neq \theta_{x | u_2}\) for all states \(x\) of \(X.\)
\end{lemma}
\begin{proof}
We first construct a parameterization \(\Theta\) for the subgraph \(X \rightarrow Y\) that induces exactly \(\Pr^\star(X,Y)\). In particular,
let \(\theta_{x} = \Pr^\star(x)\) and \(\theta_{y | x} = \Pr^\star(y | x)\) for all states \(x\) and \(y\). Note that all values in \(\Theta\) are in \((0,1)\) since \(\Pr^\star(X,Y) > 0\).

We next construct a parameterization \(\Theta'\) for the subgraph \(U \rightarrow X \rightarrow Y\) based on \(\Theta\). In particular, let \(\epsilon \rightarrow 0^+\) be an arbitrarily small constant, we assign
\begin{equation*}
\begin{split}
&\theta'_{u_1} = \theta'_{u_2} = 0.5\\
&\theta'_{x | u_1} = \theta_{x} + \epsilon\\
&\theta'_{x | u_2} = \theta_{x} - \epsilon\\
&\theta'_{y | x} = \theta_{y | x}\\
\end{split}
\end{equation*}
for all states \(x, y.\) Let \(\Pr'\) be the distribution induced by \(\Theta',\) we show that \(\Pr' = \Pr^\star\) as follows.
\begin{equation*}
\begin{split}
\Pr'(x,y) &= \theta'_{u_1}\theta'_{x|u_1}\theta'_{y|x} + \theta'_{u_2}\theta'_{x|u_2}\theta'_{y|x}\\    
&=0.5 (\theta_{x} + \epsilon)\theta_{y | x} + 0.5 (\theta_{x} - \epsilon)\theta_{y | x} \\
&=\theta_x \theta_{y | x} = \Pr^\star(x,y).\\
\end{split}
\end{equation*}
By construction, \(\Theta'\) satisfies both conditions (i) \& (ii).

We finally construct \(\Theta''\) based on \(\Theta'\) by assigning the same CPTs for \(U\) and \(X\) but different CPTs for \(Y.\) In particular,
\[\theta''_{y | u_1, x} = \theta''_{y | u_2, x} = \theta'_{y | x}\]
\(\Theta''\) also induces the distribution \(\Pr^\star\) since \((Y \indep U | X).\)
\end{proof}

\begin{proof}[Proof of Proposition~\ref{prop:strong-uid1}]
Let \(\Theta^1\) be parameterization that satisfies the conditions in Lemma~\ref{lem:string-uid1-1}, we next construct \(\Theta^2\) that also induces \(\Pr^\star(X,Y)\) but disagrees with \(\Theta^1\) on the causal effect \(\Pr_{x_1}(y_1)\). We do so by assigning same CPTs for \(U\) and \(X\), i.e., \(\Theta^2_U = \Theta^1_U\) and \(\Theta^2_X = \Theta^1_X\), but different CPTs for \(Y.\) 
Specifically, for each state \(y \neq y_1\), we assign
\begin{equation*}
\theta^2_{y | u_1, x} = 
\begin{cases}
\theta^1_{y | u_1, x} + \epsilon & \text{if } y = y_1\\
\theta^2_{y | u_1, x} = \theta^1_{y | u_1,x} - \frac{\epsilon}{m-1} & \text{if } y\neq y_1\\
\end{cases}     
\end{equation*}
and
\begin{equation*}
\theta^2_{y | u_2, x} = 
\begin{cases}
\theta^1_{y | u_2, x} - \frac{\Pr^1(u_1,x_1)}{\Pr^1(u_2, x_1)} & \text{if } y = y_1\\
\theta^1_{y | u_2,x} + \frac{\Pr^1(u_1,x_1)}{\Pr^1(u_2, x_1)} \frac{\epsilon}{(m-1)} & \text{if } y\neq y_1\\
\end{cases}     
\end{equation*}
where \(\Pr^1\) denotes the distribution induced by \(\Theta^1.\)

We first prove show that \(\Theta^2\) induces a distribution \(\Pr^2\) where \(\Pr^2(X,Y) = \Pr^1(X,Y) = \Pr^\star(X,Y).\) We first consider the instantiations that contains \(y_1,\)
\begin{equation*}
\begin{split}
&\Pr^2(x,y_1) = \theta^2_{u_1}\theta^2_{x|u_1}\theta^2_{y_1|u_1,x} + \theta^2_{u_2}\theta^2_{x|u_2}\theta^2_{y_1|u_2,x}\\    
&=\theta^1_{u_1}\theta^1_{x|u_1}(\theta^1_{y_1|u_1,x} + \epsilon) + \theta^1_{u_2}\theta^1_{x|u_2}(\theta^1_{y_1|u_2,x} - \frac{\Pr^1(u_1,x)}{\Pr^1(u_2, x)}\epsilon)\\
&=\Pr^1(x,y_1).\\
\end{split}
\end{equation*}
The last step is due to the fact that \(\theta^1_{u_1}\theta^1_{x|u_1} = \Pr(u_1,x).\)
We next consider the case when \(y \neq y_1.\)
\begin{equation*}
\begin{split}
&\Pr^2(x,y) = \theta^2_{u_1}\theta^2_{x|u_1}\theta^2_{y|u_1,x} + \theta^2_{u_2}\theta^2_{x|u_2}\theta^2_{y|u_2,x}\\    
&=\theta^1_{u_1}\theta^1_{x|u_1}(\theta^1_{y|u_1,x} - \frac{\epsilon}{m-1}) + \\
&\theta^1_{u_2}\theta^1_{x|u_2}(\theta^1_{y|u_2,x} + \frac{\Pr^1(u_1,x)}{\Pr^1(u_2, x)} \frac{\epsilon}{(m-1)})\\
&=\Pr^1(x,y).\\
\end{split}
\end{equation*}

Together with the previous result, we have \(\Pr^2(X,Y) = \Pr^\star(X,Y).\) We are left to show that \(\Theta^1\) and \(\Theta^2\) induce different causal effect, i.e., \(\Pr^1_{x_1}(y_1) \neq \Pr^2_{x_1}(y_1)\), which can be proved as follows. 
\begin{equation*}
\begin{split}
&\Pr^2_{x_1}(y_1) = \theta^2_{u_1}\theta^2_{y_1|u_1,x_1} + \theta^2_{u_2}\theta^2_{y_1|u_2,x_1}\\    
&=\theta^1_{u_1}(\theta^1_{y_1|u_1,x_1} + \epsilon) + \theta^1_{u_2}(\theta^1_{y_1|u_2,x_1} - \frac{\Pr^1(u_1,x_1)}{\Pr^1(u_2, x_1)}\epsilon)\\
&= \Pr^1_{x_1}(y_1) + \theta^1_{u_1}\epsilon - \theta^1_{u_2}\frac{\Pr^1(u_1,x_1)}{\Pr^1(u_2, x_1)}\epsilon\\
&=\Pr^1_{x_1}(y_1) + \frac{\Pr^1(u_1)(\Pr^1(x_1 | u_2)-\Pr^1(x_1 | u_1))}{\Pr^1(x_1|u_2)}\epsilon \\
&\neq \Pr^1_{x_1}(y_1).\\
\end{split}
\end{equation*}
since \(\Pr^1\) satisfies Conditions (i) \& (ii) in Lemma~\ref{lem:string-uid1-1}.
\end{proof}

\subsection*{Proof of Proposition~\ref{prop:did}}
The result follows from Definition~\ref{def:constrained-id} and Definition~\ref{def:inv-exp}.

\subsection*{Proof of Proposition~\ref{prop:id-cut2}}
\begin{proof}
By Proposition~\ref{prop:did}, it suffices to show that the output of the AC is \(\langle G, \V, \MM \rangle\)-invariant. If we replace all the expressions of nodes on the invariant-cut with their projections, it is guaranteed that the output expression of the AC (with projections) only involves variables \(\V,\) which is determined by \(\Pr(\V).\)  
\end{proof}

\subsection*{Proof of Theorem~\ref{thm:subsume-id}}
\begin{proof}
First note that the completeness proof of the ID algorithm~\citep{aaai/ShpitserP06,aaai/HuangV06} assumes all variables are binary. Hence, the if direction of the theorem must hold.

Before proving the only-if direction, we briefly review the complete causal-effect identifiability algorithm in~\citep{aaai/HuangV06}. The algorithm starts by omitting all variables that are not ancestors of \(\Y\) in the mutilated graph \(G_{\overline{\X}}\) and decomposing the mutilated graph into into c-components \(\CC = \{C_1, \dots, C_k\}.\)\footnote{\(G_{\overline{\X}}\) is obtained from the original causal graph \(G\) by removing the incoming edges of \(\X.\)} The goal is then to identify the c-factor for each of the c-component in \(\CC.\) This is implemented through a recursive procedure that identifies some \(C \in \CC\) from another graph \(T.\) In each recursive call, two operations may be performed: (i) Assign \(T = An(C)_{T} \cap T\) where \(An(C)_T\) denotes the variables in \(T\) that are ancestors of variables in \(C;\) and (ii) decompose \(T\) into c-components and identify \(C\) from one of the c-components. We say that \(C\) is identifiable from \(T\) iff \(T = C\) when the recursion terminates. The causal effect is identifiable iff all c-components in \(\CC\) can be identified from \(G.\)

One key observation is that both operations are captured by the cases in Definition~\ref{def:id-ccomp}. Operation~(i) is equivalent to recursively pruning leaf nodes that are not in \(C,\) and operation~(ii) is the same as case~(3) in Definition~\ref{def:id-ccomp}. Hence, a c-component \(C \in \CC\) is identifiable only if \(C\) is a \(\V\)-computable subgraph of \(G\), and we can always find a projection for \(C\) if it can be identified by the complete algorithm in~\citep{aaai/HuangV06}.

We now construct an AC based on the expression \(\sum_{\V \setminus (\X \cup \Y)} \prod_{C \in \CC} Q[C],\) where \(Q[C]\) denotes the c-factor for the c-component \(C\) and can be represented as an AC in a standard way. Note that the AC always computes the causal effect under all models. Moreover, if the causal effect is identifiable, we can always find projections for the AC nodes corresponding to these \(Q[C]\)'s, which form an invariant-cut for the AC.
\end{proof}

\end{document}